\documentclass[a4paper,11pt]{article}

\usepackage{natbib}
\usepackage{amsmath}
\usepackage{amssymb}
\usepackage[latin1]{inputenc}
\usepackage{graphicx,color,rotating}
\usepackage{tikz}
\usetikzlibrary{arrows,fit,backgrounds,positioning,shapes.geometric}
\usepackage{authblk}
\usepackage[margin=25mm]{geometry}
\usepackage{hyperref}

\usepackage{amsmath,amssymb}
\usepackage[latin1]{inputenc}

\usepackage{graphicx,color,rotating}

\newcommand{\indep}{{\perp\!\!\!\perp}}
\newcommand{\nindep}{\not\!\perp\!\!\!\perp}

\newcommand{\supplmat}{Supplementary material}

\newcommand{\gcite}{\citep}
\newcommand{\gcitet}{\citet}
\newcommand{\gcitep}{\citealp}

\newtheorem{theorem}{Theorem}
\newtheorem{corollary}{Corollary}
\newtheorem{proposition}{Proposition}
\newtheorem{definition}{Definition}

\newenvironment{proof}{\noindent \emph{Proof.}}{\hspace{\stretch{1}}$\square$ \\ \ }

\graphicspath{{Figures/}}

\title{Inferring the finest pattern of mutual independence from data}

\author{Guillaume Marrelec\textsuperscript{1,*} and Alain Giron\textsuperscript{1,*}\\
\small \textsuperscript{1} Laboratoire d'imagerie biom{\'e}dicale, LIB, Sorbonne Universit{\'e}, CNRS, INSERM, F-75006, Paris, France\\
\small \textsuperscript{*} Email: firstname.lastname@inserm.fr}

\date{}

\begin{document}

\maketitle

\normalsize

\begin{abstract}
  For a random variable $X$, we are interested in the blind extraction of its finest mutual independence pattern $\mu ( X )$. We introduce a specific kind of independence that we call dichotomic. If $\Delta ( X )$ stands for the set of all patterns of dichotomic independence that hold for $X$, we show that $\mu ( X )$ can be obtained as the intersection of all elements of $\Delta ( X )$. We then propose a method to estimate $\Delta ( X )$ when the data are independent and identically (i.i.d.) realizations of a multivariate normal distribution. If $\hat{\Delta} ( X )$ is the estimated set of valid patterns of dichotomic independence, we estimate $\mu ( X )$ as the intersection of all patterns of $\hat{\Delta} ( X )$. The method is tested on simulated data, showing its advantages and limits. We also consider an application to a toy example as well as to experimental data.
  \par
  \
\end{abstract}

\section{Introduction}

In probability theory, $n$ random variables $ X_1$, \dots, $X_n$ are said to be mutually independent if their joint distribution can be expressed as the product of their marginal distributions
\gcite[Section~2.6]{Hogg-2004}:
\begin{equation} \label{eq:def:mi}
 \Pr (X_1, \dots, X_n ) = \prod_{ i = 1 } ^ n \Pr ( X_i ).
\end{equation}
Analysis of mutual independence is a key issue in statistics. Several subtopics relevant to this problem have been examined in depth. Some authors have provided efficient measures to test for the independence between two variables in various conditions, the two variables being either unidimensional (\gcitep{Spearman-1904}; \gcitep{Hotelling-1936}; \gcitep{Kendall-1938}; \gcitep{Gebelein-1941}; \gcitep{Hoeffding-1948}; \gcitep{Renyi-1959b}; \gcitep{Reshef-2011}) or multidimensional (\gcitep{Jupp-1980}; \gcitep[Chaps. 2 and 8]{Cover_TM-1991}; \gcitep{Bakirov-2006}; \gcitep{Schott-2008}; \gcitep{Jiang_D-2012}; \gcitep{Szekely-2013b}). Testing for the existence of a specific pattern of mutual independence has also been a topic of interest (\gcitep[Chap. 9]{Anderson_TW-1958}; \gcitep[Chap.~8, Section 2 and 3.1, and pp. 306--307]{Kullback-1968}; \gcitep[Section~23.8]{Zar-2010}). Some researchers have focused on total, or complete, independence, i.e., mutual independence between all variables \gcite{Csorgo-1985, Schott-2005, Pfister-2018}. Finally, there has also been some interest for the investigation of mutual independence through multiple bivariate pairwise independence tests \gcite{Mao-2017, Mao-2018}.
\par
We are here interested in yet another subtopic of mutual independence analysis, namely, blind extraction of mutual independence patterns. Indeed, from the perspective of $X = \{ X_1, \dots, X_n \}$, Equation~\eqref{eq:def:mi} corresponds to a particular case---that of total independence. In the more general case, other situations may occur, where different groups of variables would be mutually independent. Investigating patterns of mutual independence on $X$ therefore requires to propose methods that extract these groups. The objective is then to explore the whole set of mutual independence patterns that could exist within a multidimensional variable and determine the ones that best explain the data. To the best of the authors' knowledge, this topic has generated few publications, with the exception of \gcitet{Marrelec-2021b}, who proposed a Bayesian approach coupled with an MCMC exploration of the pattern space.
\par
In blind extraction, a major issue that still needs to be tackled is that a collection $X$ of variables can (and, usually, does) have several patterns of mutual independence. Denoting by $\Pi ( X )$ the set of all patterns of mutual independence on $X$, we can define the \emph{finest} pattern of mutual independence $\mu ( X )$ as the intersection of all patterns of $\Pi ( X )$. In the present manuscript, we are interested in providing a one-step data-driven procedure that infers $\mu ( X )$. To this end, we introduce a particular kind of independence which we call dichotomic. A pattern of dichotomic independence on $X$ deals with the independence between a subvariable of $X$ and its complement in $X$. We denote by $\Delta ( X )$ the set of all patterns of dichotomic independence on $X$, which is a subset of $\Pi ( X )$. Our main result is that $\mu ( X )$ can be exactly reconstructed as the intersection of all elements of $\Delta ( X )$ (see Theorem~\ref{th:Delta} below). Note that, while intuitively clear, the concepts introduced above---$\Pi ( X )$, $\mu ( X )$, $\Delta ( X )$ and the notion of intersection---will be specified below using the bijection between patterns of mutual independence and partitions, together with the lattice structure of partitions. We then propose a statistical procedure that estimates $\Delta ( X )$ in the case where $X$ follows a multivariate normal distribution and the data is composed of independent and identically (i.i.d.) realizations of $X$. The approach relies on testing the minimum discrimination information statistics \gcite[Chap.~12, Section~3.6]{Kullback-1968} corresponding to all patterns of dichotomic independence and correcting for multiple comparison by controlling the false discovery rate \gcite{Benjamini-1995}. Patterns that cannot be rejected are said to belong to the estimate $\hat{\Delta} ( X )$ of $\Delta ( X )$. We finally estimate $\mu ( X )$ by $\hat{\mu} ( X )$, the intersection of all elements of $\hat{\Delta} ( X )$.
\par
The outline of the manuscript is the following. In Section~\ref{s:mipl}, we relate mutual independence and partitions, introduce the lattice structure over the set of partitions, and investigate its implications for mutual independence. In Section~\ref{s:dpi}, we define dichotomic independence, prove that the finest pattern of mutual independence can be uniquely extracted from dichotomic independence, and provide for a statistical procedure that extracts the set of patterns of dichotomic independence when the data are independent and identically distributed (i.i.d.) as a multivariate normal distribution. In Section~\ref{s:ss}, we perform a simulation study to assess the quality of the inference process as well as its strengths and weaknesses. In Section~\ref{s:hiv}, we provide an application to a toy example and, in Section~\ref{s:rd}, to real data consisting of brain recordings. Further issues are discussed in Section~\ref{s:disc}.

\section{Mutual independence, partitions and lattices} \label{s:mipl}

After a quick introduction of the notations (Section~\ref{ss:not}), we emphasize the key connection between mutual independence and partitions (Section~\ref{ss:mip}). In Sections \ref{ss:amop} and \ref{ss:tsomip}, we then focus on $\Pi ( X )$, the set of all patterns of mutual independence that hold on $X$, and $\mu ( X )$, the finest pattern of $\Pi ( X )$, providing a characterization of both concepts in terms of lattice structure.

\subsection{Notations} \label{ss:not}

We henceforth rely on the following notations. Let $X = \{ X_1, \dots, X_n \}$ be a collection of random variables and $N = [n] = \{ 1, \dots, n \}$ the index set containing the full set of suffices.  If $a = \{ i_1, \dots, i_k \}$ is a subset of $N$, then the random variable $X_a$ is defined as the subset of variables of $X$ such that
$$X_a = \{ X_{i_1}, \dots, X_{i_k} \} = \{ X_i: i \in a \}.$$
The full variable is $X_N = X$, $X_{\emptyset}$ is empty, and $X_{N \setminus a}$ denotes the subvariable of $X$ obtained by excluding $X_a$. $X_a$ can be thought of as the restriction of $X$ (treated as a mapping from $N$) to $a$. Such notations are similar to ones already existing (\gcitep{Darroch-1980}; \gcitep[Section~1.4]{Whittaker-1990}), but modified to assure consistency. 

\subsection{Mutual independence and partitions} \label{ss:mip}

There exists a bijection between patterns of mutual independence and partitions. Indeed, if $X$ is such that
\begin{equation} \label{eq:def:mi:gen}
 \Pr ( X ) = \prod_{ i = 1 } ^ k \Pr ( X_{ a_i } ),
\end{equation}
then the $a_i$'s form a partition of $N = \{ 1, \dots, n \}$, which we also express as $a_1 \mid \dots \mid a_k$. The $a_i$'s are called the blocks of the partition. The definition of a decomposition of $X$ into mutually independent subsets of variables is therefore equivalent to the choice of a partition. We let $\Omega ( N )$ be the set of all partitions of $N$. The cardinality of this set, i.e., the number of partitions of $N$, is given by the $n$th Bell number, traditionally denoted $\varpi_n$ \gcite{Rota-1964}---see Table~\ref{tab:bell-stirling} below for a few examples. The number of partitions of $N$ with exactly $k$ blocks is given by the Stirling number of the second kind ${ n \brace k }$. In the following, we will mostly focus on a  partition representation of patterns of mutual independence. In particular, we will take advantage of the bijection between patterns of mutual independence and partitions to identify both. For instance, we will say that the pattern of mutual independence is $a_1 \mid \dots \mid a_k$ as a shortcut for the fact that the pattern of mutual independence of interest is associated with the partition $a_1 \mid \dots \mid a_k$, and is therefore that $X_{ a_1 }, \dots, X_{ a_k }$ are mutually independent.

\subsection{A multiplicity of patterns} \label{ss:amop}

Until now, we have only assumed the existence of one pattern of mutual independence between subvariables of $X$. Note however that the existence of one pattern usually entails the existence of coarser patterns. Consider for instance a variable $X = X_{[6]}$ such that $X_1$, $X_{ \{ 2, 3 \} }$, $X_{ \{ 4, 5 \} }$ and $X_6$ are mutually independent, i.e., the pattern of mutual independence is associated with the partition $\pi = 1 \mid 2 3 \mid 4 5 \mid 6$. Such a pattern of mutual independence entails other coarser patterns, such as the ones associated with the partitions $1 2 3 \mid 4 5 \mid 6$ and $1 \mid 2 3 4 5 \mid 6$. Our starting point is that, given a pattern of mutual independence, it is possible to characterize other (coarser) patterns.

\begin{proposition} \label{prop:im}
 Assume that $a_1 \mid \dots \mid a_k $ is a partition of $N$ such that $X_{ a_1 }$, \dots, $X_{ a_ k }$ are mutually independent. Let $b_1, \dots, b_l$ be disjoint subsets of $N$ such that no two $b_i$'s intersect the same $a_j$ (if $b_i \cap a_j \neq \emptyset$, then $b_{i'} \cap a_j = \emptyset$ for all $i' \neq i$). Then $X_{ b_1 }$, \dots, $X_{ b_l }$ are mutually independent.
\end{proposition}

See Appendix~\ref{an:pr:prop:im} for a proof. We therefore need to consider the set of all patterns of mutual independence that are valid for a given random variable.

\begin{definition}
 For a given random variable, $\Pi ( X )$ is the subset of partitions corresponding to all existing patterns of mutual independence that hold on $X$.
\end{definition}

$\Pi ( X )$ is not empty, as even for a variable with no mutual independence, the 1-block partition $1 \dots n$ belongs to $\Pi ( X )$. Two examples of $\Pi ( X )$ are given in Figure~\ref{fig:Pi:ex}.
  
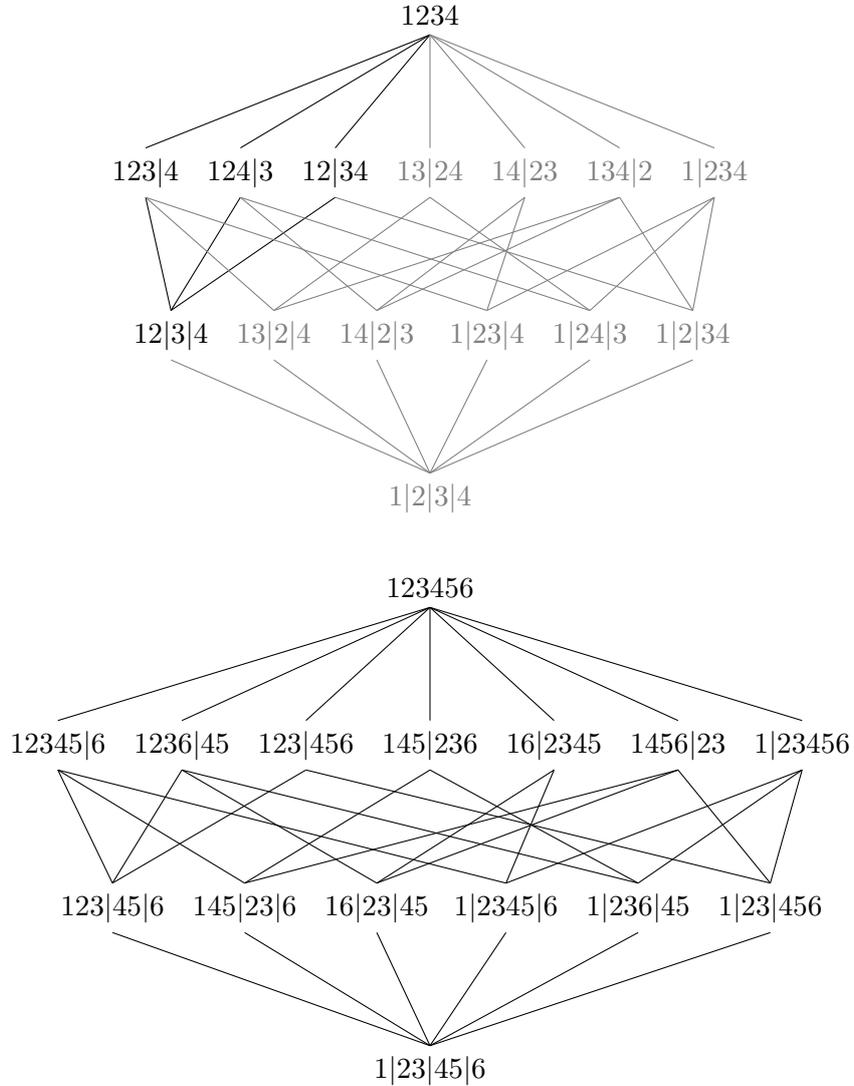
\begin{figure}[!htbp]
 \centering
 \begin{tabular}{c}
  \begin{tabular}{c}
   \begin{tikzpicture}
    \node(N1) [color=gray] {$1|2|3|4$};
    \node(N23) [above=1.5cm of N1, xshift=-0.7cm, color=gray] {$14|2|3$};
    \node(N22) [left=0.1cm of N23, color=gray] {$13|2|4$};
    \node(N21) [left=0.1cm of N22] {$12|3|4$};
    \node(N24) [above=1.5cm of N1, xshift=0.75cm, color=gray] {$1|23|4$};
    \node(N25) [right=0.1cm of N24, color=gray] {$1|24|3$};
    \node(N26) [right=0.1cm of N25, color=gray] {$1|2|34$};
    \node(N34) [above=1.5cm of N23, xshift=0.7cm, color=gray] {$13|24$};
    \node(N33) [left=0.1cm of N34] {$12|34$};
    \node(N32) [left=0.1cm of N33] {$124|3$};
    \node(N31) [left=0.1cm of N32] {$123|4$};
    \node(N35) [right=0.1cm of N34, color=gray] {$14|23$};
    \node(N36) [right=0.1cm of N35, color=gray] {$134|2$};
    \node(N37) [right=0.1cm of N36, color=gray] {$1|234$};
    \node(N4) [above=1.5cm of N34] {$1234$};
    \draw[gray](N1.north)--(N21.south);
    \draw[gray](N1.north)--(N22.south);
    \draw[gray](N1.north)--(N23.south);
    \draw[gray](N1.north)--(N24.south);
    \draw[gray](N1.north)--(N25.south);
    \draw[gray](N1.north)--(N26.south);
    \draw(N21.north)--(N31.south);
    \draw(N21.north)--(N33.south);
    \draw(N21.north)--(N32.south);
    \draw[gray](N22.north)--(N31.south);
    \draw[gray](N22.north)--(N34.south);
    \draw[gray](N22.north)--(N36.south);
    \draw[gray](N23.north)--(N32.south);
    \draw[gray](N23.north)--(N35.south);
    \draw[gray](N23.north)--(N36.south);
    \draw[gray](N24.north)--(N31.south);
    \draw[gray](N24.north)--(N35.south);
    \draw[gray](N24.north)--(N37.south);
    \draw[gray](N25.north)--(N32.south);
    \draw[gray](N25.north)--(N34.south);
    \draw[gray](N25.north)--(N37.south);
    \draw[gray](N26.north)--(N33.south);
    \draw[gray](N26.north)--(N36.south);
    \draw[gray](N26.north)--(N37.south);
    \draw(N31.north)--(N4.south);
    \draw(N32.north)--(N4.south);
    \draw(N33.north)--(N4.south);
    \draw[gray](N34.north)--(N4.south);
    \draw[gray](N35.north)--(N4.south);
    \draw[gray](N36.north)--(N4.south);
    \draw[gray](N37.north)--(N4.south);
   \end{tikzpicture}
  \end{tabular} \\
  \ \\
  \begin{tabular}{c}
   \begin{tikzpicture}
    \node(N1) {$1|23|45|6$};
    \node(N23) [above=1.5cm of N1, xshift=-0.7cm] {$16|23|45$};
    \node(N22) [left=0.1cm of N23] {$145|23|6$};
    \node(N21) [left=0.1cm of N22] {$123|45|6$};
    \node(N24) [above=1.5cm of N1, xshift=1cm] {$1|2345|6$};
    \node(N25) [right=0.1cm of N24] {$1|236|45$};
    \node(N26) [right=0.1cm of N25] {$1|23|456$};
    \node(N34) [above=1.5cm of N23, xshift=0.7cm] {$145|236$};
    \node(N33) [left=0.1cm of N34] {$123|456$};
    \node(N32) [left=0.1cm of N33] {$1236|45$};
    \node(N31) [left=0.1cm of N32] {$12345|6$};
    \node(N35) [right=0.1cm of N34] {$16|2345$};
    \node(N36) [right=0.1cm of N35] {$1456|23$};
    \node(N37) [right=0.1cm of N36] {$1|23456$};
    \node(N4) [above=1.5cm of N34] {$123456$};
    \draw(N1.north)--(N21.south);
    \draw(N1.north)--(N22.south);
    \draw(N1.north)--(N23.south);
    \draw(N1.north)--(N24.south);
    \draw(N1.north)--(N25.south);
    \draw(N1.north)--(N26.south);
    \draw(N21.north)--(N31.south);
    \draw(N21.north)--(N33.south);
    \draw(N21.north)--(N32.south);
    \draw(N22.north)--(N31.south);
    \draw(N22.north)--(N34.south);
    \draw(N22.north)--(N36.south);
    \draw(N23.north)--(N32.south);
    \draw(N23.north)--(N35.south);
    \draw(N23.north)--(N36.south);
    \draw(N24.north)--(N31.south);
    \draw(N24.north)--(N35.south);
    \draw(N24.north)--(N37.south);
    \draw(N25.north)--(N32.south);
    \draw(N25.north)--(N34.south);
    \draw(N25.north)--(N37.south);
    \draw(N26.north)--(N33.south);
    \draw(N26.north)--(N36.south);
    \draw(N26.north)--(N37.south);
    \draw(N31.north)--(N4.south);
    \draw(N32.north)--(N4.south);
    \draw(N33.north)--(N4.south);
    \draw(N34.north)--(N4.south);
    \draw(N35.north)--(N4.south);
    \draw(N36.north)--(N4.south);
    \draw(N37.north)--(N4.south);
   \end{tikzpicture}
  \end{tabular}
 \end{tabular}
 \caption{\textbf{Two examples of $\Pi ( X )$.} Top: representation of $\Pi ( X )$ corresponding to $X = \{ X_1, X_2, X_3, X_4 \}$ such that $X_{ \{ 1, 2 \} }$, $X_3$ and $X_4$ are mutually independent, i.e., $\mu ( X ) = 1 2 \mid 3 \mid 4$. $\Pi ( X )$ is superimposed on $\Omega ( [4] )$ (in gray). Bottom: representation of $\Pi ( X )$ corresponding to $X = \{ X_1, X_2, X_3, X_4, X_5, X_6 \}$ such that $X_1$, $X_{ \{ 2, 3 \} }$, $X_{ \{4, 5 \} }$ and $X_6$ are mutually independent, i.e.,  $\mu ( X ) = 1 \mid 2 3 \mid 4 5 \mid 6$. Only $\Pi ( X )$ is represented.} \label{fig:Pi:ex}
\end{figure}


Importantly, the existence of a pattern of mutual independence does not prevent the existence of finer patterns either. Going back to our example above, stating that $X_1$, $X_{ \{ 2, 3 \} }$, $X_{ \{ 4, 5 \} }$ and $X_6$ are mutually independent (corresponding to partition $\pi = 1 \mid 23 \mid 45 \mid 6$) is not incompatible with the fact that $X_1$, $X_{ \{ 2, 3 \} }$, $X_4$, $X_5$ and $X_6$ could also be mutually independent (corresponding to partition $\pi' = 1 \mid 2 3 \mid 4 \mid 5 \mid 6$). For this reason,  we also need to define the notion of finest pattern of mutual independence.

\begin{definition}
 Let $X$ be a random variable. $\mu ( X )$ is the partition that can be associated with the finest pattern of mutual independence for $X$.
\end{definition}

The existence and unicity of this finest pattern are proved in the next section.

\subsection{The sublattice of mutual independence patterns} \label{ss:tsomip}

At this point, it is important to note that $\Omega ( N )$ has the key property of being a lattice (\gcitep{Birkhoff-1935b}; \gcitep[Example~9, pp.~15--16]{Birkhoff-1973}; \gcitep[Chap.~I, Section~2.B]{Aigner-1979}; for a quick review, see Section~1 of \supplmat). As a consequence, it is associated with a partial order, denoted ``$\leqslant$''.  We say that, for two partitions $\pi_1$ and $\pi_2$, $\pi_1 \leqslant \pi_2$ if $\pi_1$  is finer than $\pi_2$, i.e., each block of $\pi_1$ is contained in a block of $\pi_2$. For example, if $\pi$ is defined as in the previous section, $\pi = 1 \mid 2 3 \mid 4 5 \mid 6$, we have $\pi \leqslant 1 2 3 \mid 4 5 \mid 6$, $\pi \leqslant 1 \mid 2 3 4 5 \mid 6$ and $1 \mid 2 3 \mid 4 \mid 5 \mid 6 \leqslant \pi$. There is actually a particular relation between the partial order and patterns of mutual independence, which is a direct consequence of Proposition~\ref{prop:im}:

\begin{corollary} \label{prop:im:pfq}
 If $\pi$ is a partition associated with a pattern of mutual independence on $X$, then any partition $\pi'$ such that $\pi \leqslant \pi'$ is also associated with a pattern of mutual independence on $X$.
\end{corollary}

Since $\Omega ( N )$ is a lattice, the partial order $\leqslant$ can be used to define for every pair of elements their unique least upper bound, or join (``union of two partitions'', denoted $\vee$), and their unique greatest lower bound, or meet (``intersection of two partitions'', denoted $\wedge$). While the terms ``join'' and ``meet'' are more adapted to the situation, we will rather use the terms ``union'' (instead of ``join'') and ``intersection'' (instead of ``meet''), which are more pictural and take advantage of the analogy with sets. Similarly, while the results will be stated and proved in lattice terminology in the appendix, we will try in the main text to restrict
ourselves to results that can be stated in general terms.
\par
Interestingly, we can show that the set of patterns of mutual independence is stable by union and intersection (i.e., the union and intersection of any pair of patterns of mutual independence on $X$ are patterns of mutual independence on $X$ as well) and that the finest pattern has a simple characterization (see Appendix~\ref{an:th:Pi} for a proof)
  
\begin{theorem} \label{th:Pi}
 Let $\Pi ( X )$ be the subset of partitions corresponding to all patterns of mutual independence that hold on $X$. Then $\Pi ( X )$ is a sublattice of $\Omega ( N )$. Both its coarsest and its finest elements exist and are unique: Its coarsest element is the trivial one-block partition $1 \dots n$, while its finest element is equal to
 \begin{equation} \label{eq:def:mu}
  \mu ( X ) = \wedge_{ \pi \in \Pi ( X ) } \pi.
 \end{equation}
\end{theorem}

In words, the partition $\mu ( X )$ corresponding to the finest pattern of mutual independence on $X$ is equal to the finest partition on $\Pi ( X )$. 
\par
For instance, considering the case $n = 4$, there are 15 potential patterns of mutual independence (see Figure~\ref{fig:Pi:ex}, top). If $X$ is such that $X_{123} \indep X_4$, $X_{124} \indep X_3$, $X_{12} \indep X_{34}$ as well as $X_{12}$, $X_3$, $X_4$ mutually independent, then $\Pi ( X ) = \{ 123 \mid 4, 124 \mid 3, 12 \mid 34, 12 \mid 3 \mid 4 \}$. $\mu ( X ) $ is obtained as the intersection of all elements of $\Pi ( X )$, which is $12 \mid 3 \mid 4$ (see again Figure~\ref{fig:Pi:ex}, top). This means that the finest pattern of mutual independence on $X$ is that $X_{12}$, $X_3$ and $X_4$ are mutually independent.

\section{Dichotomic independence} \label{s:dpi}

In this section, we first introduce the notion of dichotomic independence (Section~\ref{ss:dpi:def}) and relate it to the finest pattern of mutual independence (Section~\ref{ss:dpi:rel}). We then provide an inference procedure to extract patterns of dichotomic independence and estimate $\mu ( X )$ (Section~\ref{ss:dpi:inf}).

\subsection{Definition} \label{ss:dpi:def}

We start by formally introducing dichotomic independence.

\begin{definition}
 A pattern of independence is dichotomic if it holds between a subset of variables $X_a$  and the remaining variables $X_{ N \setminus a }$. 
\end{definition}

A pattern of dichotomic independence $X_a \indep X_{ N \setminus a }$ is therefore characterized by a formula of the form
\begin{equation} \label{eq:dicho}
 \Pr ( X ) = \Pr ( X_a ) \, \Pr ( X_{ N \setminus a } ).
\end{equation} 
It can therefore be associated with a 2-block partition $a \mid ( N \setminus a )$ of $N$, also called dichotomy, or bipartition. We denote by $\Omega_2 ( N )$ the set of all bipartitions on $N$.

\subsection{Relating mutual and dichotomic independence} \label{ss:dpi:rel}

A pattern of mutual independence is characterized by a specific set of patterns of dichotomic independence, as expressed in the following result.

\begin{proposition} \label{pr}
 Let $\mu ( X ) = a_1 \mid \dots \mid a_k$ be the finest pattern of mutual independence on $X$. Then the set of patterns of dichotomic independence on $X$ is given by
 \begin{equation} \label{eq:def:Delta}
  \Delta ( X ) = \Pi ( X ) \cap \Omega_2 ( N ) = \left\{ \pi = \cup_{ j \in b_1 } a_j \mid \cup_{ j \in b_2 } a_j, \quad b_1 \mid b_2 \in \Omega_2 ( [ k ] ) \right\}
 \end{equation}
\end{proposition}

\begin{proof}
 It is obvious that $\Delta ( X ) = \Pi ( X ) \cap \Omega_2 ( N )$. The second expression of $\Delta ( X )$ is a direct consequence of Proposition~\ref{prop:im:pfq} for a partition $\pi \in \Omega_2 ( N )$.
\end{proof}

In other words, the finest pattern of mutual independence $\mu ( X )$ between $k$ subvariables of $X$ entails all patterns of dichotomic independence that can be obtained by separating the $k$ blocks of $\mu ( X )$ in two. As a consequence the number of patterns of dichotomic independence entailed by a given $\mu ( X )$ composed of $k$ blocks is given by the cardinality of $\Omega_2 ( [ k ] )$, i.e., the number of partitions of $[ k ]$ into two blocks, which is Stirling number of the second kind ${ k \brace 2 } = 2 ^ { k - 1 } - 1$. 
\par
We are now in position to express $\mu ( X )$ from $\Delta ( X )$.

\begin{theorem} \label{th:Delta}
   Let $\mu ( X )$ be the finest pattern of mutual independence of $X$ and $\Delta ( X )$ the set of patterns of dichtomic independence on $X$. Then we have 
 \begin{equation}
   \mu ( X ) = \wedge_{ \delta \in \Delta ( X ) } \delta.
 \end{equation}
\end{theorem}

This result is proved in Appendix~\ref{an:th:Delta}. In words, the finest pattern of mutual independence is given by the intersection of all patterns of dichotomic independence. Having access to the set of patterns of dichotomic independence is therefore enough to reconstruct $\mu ( X )$. In practice, $\mu ( X )$ is composed of blocks such that each block is included in a block of each dichotomic partition. In other words, two variables belong to the same block of $\mu ( X )$ if and only if they belong to the same block for all dichotomic partitions. 
\par
For instance, going back to the case $n = 4$ and Figure~\ref{fig:Pi:ex}, top, there are 7 potential patterns of dichotomic independence. We assume that, among these 7 patterns, only 3 hold for $X$: $X_{123} \indep X_4$, $X_{124} \indep X_3$, and $X_{12} \indep X_{34}$; besides, $X_{13} \nindep X_{24}$, $X_{14} \nindep X_{23}$, $X_{134} \nindep X_2$ and $X_1 \nindep X_{234}$. We then obtain $\Delta ( X ) = \{ 123 \mid 4, 124 \mid 3, 12 \mid 34 \}$ and $\mu ( X )$ can be retrieved as the intersection of all elements of $\Delta ( X )$. More precisely, since 1 and 2 belong to the same block for all three partitions of $\Delta ( X )$ ($\underline{12}3 \mid 4$, $\underline{12}4 \mid 3$, $\underline{12} \mid 34$), they will also belong to the same block in $\mu ( X )$. By contrast, 1 and 3 belong to two different blocks in $124 \mid 3$, so they will belong to two different blocks in $\mu ( X )$. Since in this partition, 3 also belongs to a different block than 2 and 4, it will form a block in itself in $\mu ( X )$. The same argument holds for 4 and partition $123 \mid 4$, leading to 4 as a 1-variable block in $\mu ( X )$. In the end, the intersection is given by $\mu ( X ) = 12 \mid 3 \mid 4$.

\subsection{Inference} \label{ss:dpi:inf}

In the previous section, we translated the problem of finding $\mu ( X )$ into the question of retrieving $\Delta ( X )$. We here provide a statistical test that estimates $\Delta ( X )$ from data when $X$ follows a multivariate normal distribution. To this aim, we rely on the minimum discrimination information statistic \gcite[Chap.~12, Section~3.6]{Kullback-1968} to simultaneously test for the existence of every potential pattern of dichotomic independence, with a correction for multiple comparisons by controlling the false discovery rate \gcite{Benjamini-1995}. The estimate $\hat{\Delta} ( X )$ of $\Delta ( X )$ is obtained as the set of all patterns of dichotomic independence that cannot be rejected. Finally, the estimate $\hat{\mu} ( X )$ for the finest pattern $\mu ( X )$ is obtained as
\begin{equation}
  \hat{\mu} ( X ) = \wedge_{ \delta \in \hat{\Delta} ( X ) } \delta.
\end{equation}
The test is described in more details below.

\subsubsection{Testing for the existence of one pattern of dichotomic independence} \label{sss:test}

We here describe the minimum discrimination information statistic \gcite[Chap.~12, Section~3.6]{Kullback-1968} used to test for the existence of each pattern of dichotomic independence.
\par
Assume that the data are composed of $k$ independent and identically distributed (i.i.d.) samples from a multivariate normal distribution with mean $\mu$ and covariance matrix $\Sigma$. Let $a \subset N$ and $\bar{a} = N \setminus a$; $n_a$ and $n_{ \bar{a} }$ be the cardinalities of $a$ and $\bar{a}$, respectively. The bipartition $a \mid \bar{a}$ generates a natural partition of $\Sigma$ into
\begin{equation}
 \Sigma = \begin{pmatrix}
  \Sigma_{ a a } & \Sigma_{ a \bar{a} }  \\
  \Sigma_{ \bar{a} a } & \Sigma_{ \bar{a} \bar{a} }
 \end{pmatrix},
\end{equation}
where $\Sigma_{ a a }$ is the $n_a$-by-$n_a$ covariance matrix of $X_a$,  $\Sigma_{ \bar{a} \bar{a} }$ the $n_{\bar{a}}$-by-$n_{\bar{a}}$ covariance matrix of $X_{\bar{a}}$, $\Sigma_{ a \bar{a} }$ the $n_a$-by-$n_{\bar{a}}$ covariance between $X_a$ and $X_{\bar{a}}$, and $\Sigma_{ \bar{a} a }$ its transpose. A pattern of dichotomic independence of the form $X_a \indep X_{ \bar{a} }$ can be tested by considering the null hypothesis that $\Sigma$ is block-diagonal 
\begin{equation}
 H_{ a \mid \bar{a} }: \qquad \Sigma = \begin{pmatrix}
  \Sigma_{ a a } & 0 \\
  0 & \Sigma_{ \bar{a} \bar{a} }
 \end{pmatrix},
\end{equation}
i.e., $\Sigma_{ a \bar{a} } = 0$ and $\Sigma_{ \bar{a}  a } = 0$. Let $R$ be the sample correlation matrix, partitioned as $\Sigma$,
\begin{equation}
 R = \begin{pmatrix}
  R_{ a a } & R_{ a \bar{a} } \\
  R_{ \bar{a} a } & R_{ \bar{a} \bar{a} }
 \end{pmatrix}.
\end{equation}
We then use the fact that, under $H_{ a \mid \bar{a} }$, the minimum discrimination information statistic against $H_{ a \mid \bar{a} }$
\begin{equation} \label{eq:mdis:def}
  2 \hat{I}_{ a \mid \bar{a} } = ( k - 1 ) \ln \frac{ \det ( R_{ a a } ) \, \det ( R_{ \bar{a} \bar{a} } ) } { \det ( R ) }
\end{equation}
has a distribution that is approximately noncentral chi-squared with
\begin{equation*}
 \frac{ n ( n + 1 ) } { 2 } - \frac{ n_a ( n_a + 1 ) } { 2 } - \frac{ n_{ \bar{a} } ( n_{ \bar{a} } + 1 ) } { 2 } = n_a n_{ \bar{a} }
\end{equation*}
 degrees of freedom and noncentrality parameter
\begin{equation*}
 \frac{ 1 }{ 12 ( k - 1 ) } \left[ ( 2 n ^ 3 + 3 n ^ 2 - n ) -( 2 n_a ^ 3 + 3 n_a ^ 2 - n_a )- ( 2 n_{ \bar{a} } ^ 3 + 3 n_{ \bar{a} } ^ 2 - n_{ \bar{a} } )\right].
\end{equation*}
Asymptotically, this is chi-squared distributed with the same number of degrees of freedom \gcite[Chap.~12, Section~3.6]{Kullback-1968}. This makes it possible to calculate a $p$-value corresponding to the null hypothesis $H_{ a \mid \bar{a} }$ that $X_a \indep X_{ \bar{a} }$. This test is consistent, as its power tends to 1 as the size of the dataset $k$ tends to infinity \gcite[Chap.~5, Section~5]{Kullback-1968}.
\par
Note that, for all subsequent analyses, the asymptotic chi-squared distribution will be used. We will come back to this point in the discussion.

\subsubsection{Multiple comparison}

To estimate $\Delta ( X )$, we need to simultaneously test 
\begin{equation}
 m = { n \brace 2 } = 2 ^ { n - 1 } - 1
\end{equation}
patterns of dichotomic independence using the method detailed above, each pattern $X_a \indep X_{ \bar{a} }$ being associated with its own null hypothesis $H_{ a \mid \bar{a} }$. This multiple comparison procedure can be conducted by controlling the false discovery rate (FDR). For a given significance level $\alpha$, the FDR-controlling approach finds the largest $m_{\mathrm{thres}}$ such that the $m_{\mathrm{thres}}$th smallest $p$-value is smaller than $\alpha m_{\mathrm{thres}} / m$ \gcite{Benjamini-1995}. These $m_{\mathrm{thres}}$ smallest $p$-values are declared significant and the corresponding null hypotheses are rejected. The remaining $m - m_{\mathrm{thres}}$ hypotheses are then assumed to hold, and $\hat{\Delta} ( X )$ is composed of the corresponding patterns.

\subsubsection{Summary of inference process} \label{ss:soip}

To summarize the inference process, each pattern of dichotomic independence $X_a \indep X_{\bar{a}}$ is associated with a null hypothesis $H_{ a \mid \bar{a} }$ and a minimum discrimination information statistic $\hat{I}_{ a \mid \bar{a} }$ as in Equation~(\ref{eq:mdis:def}). Since the distribution of $\hat{I}_{ a \mid \bar{a} }$ under $H_{ a \mid \bar{a} }$ is known asymptotically, a $p$-value can be computed. For a given significance level $\alpha$, the ${ n \brace 2 }$ patterns can be simultaneously tested while controlling the false discovery rate (FDR). The $m_{\mathrm{thres}}$ patterns with $p$-values smaller than $\alpha m_{\mathrm{thres}} / m$ are rejected, while the remaining $m - m_{\mathrm{thres}}$ are kept to form $\hat{\Delta} ( X )$. Finally, $\hat{\mu} ( X )$ is obtained as the intersection of all patterns in $\hat{\Delta} ( X )$.
\par
For instance, in the case $n = 4$ and related Figure~\ref{fig:Pi:ex}, top, if the four patterns $13 \mid 24$, $14 \mid 23$, $134 \mid 2$ and $1 \mid 234$ are rejected as significant (corresponding to $X_{13} \nindep X_{24}$, $X_{14} \nindep X_{23}$, $X_{134} \nindep X_2$ and $X_1 \nindep X_{234}$), then the remaining three patterns are deemed nonsignificant (i.e., $X_{123} \indep X_4$, $X_{124} \indep X_3$, $X_{12} \indep X_{34}$). We obtain $\hat{\Delta} ( X ) = \{ 123 \mid 4, 124 \mid 3, 12 \mid 34 \}$ and 
\begin{equation}
 \hat{\mu} ( X ) =  123 \mid 4 \wedge 124 \mid 3 \wedge 12 \mid 34 = 12 \mid 3 \mid 4.
\end{equation}
\par
Consistency of the whole procedure (simultaneous tests and correction for multiple comparisons using the FDR) is a consequence of the consistency of the individual tests together with the fact that the number of simultaneous tests is a function of the number of variables $n$ and is therefore fixed,  while the data size $k$ tends to infinity (see Appendix~\ref{an:consist} for a proof).

\subsubsection{Positive versus negative cases}

In the usual terminology of binary classification, each case for which the null hypothesis $H_{ a \mid \bar{a} }$ does not hold (corresponding to $X_a \nindep X_{\bar{a}}$) is coined ``positive'', while each case for which the null hypothesis $H_{ a \mid \bar{a} }$ does hold (corresponding to $X_a \indep X_{\bar{a}}$) is termed ``negative''.
\par
To validate the result of a classification procedure in the face of a ground truth, the notions of true/false positive/negative are also useful. A positive case in the ground truth is a true positive if it is correctly detected as a positive case by the classification procedure; otherwise, it is wrongly detected as a negative case by the classification procedure and is termed a false negative. Similarly, a negative case in the ground truth is a true negative if it is correctly detected as a negative case by the classification procedure; otherwise, it is wrongly detected as a positive case by the classification procedure and is termed a false positive.
\par
For instance, still in the case $n = 4$ and related Figure~\ref{fig:Pi:ex}, top, discussed in Section~\ref{ss:soip}, the positive cases are the four patterns $X_{13} \nindep X_{24}$, $X_{14} \nindep X_{23}$, $X_{134} \nindep X_2$, and $X_1 \nindep X_{234}$, and the negative cases $X_{123} \indep X_4$, $X_{124} \indep X_3$, and $X_{12} \indep X_{34}$. If the classification correctly determines that $X_{13} \nindep X_{24}$, then it is a true positive; otherwise, the classification procedure concludes that $X_{13} \indep X_{24}$ and it is a false negative. Symmetrically, if the classification correctly determines that $X_{123} \indep X_4$, then it is a true negative; otherwise, the classification procedure concludes that $X_{123} \nindep X_4$ and it is a false positive.

\section{Simulation study} \label{s:ss}

To assess the behavior of the method, we performed a simulation study with $n = 6$ variables, corresponding to $\varpi_6 = 203$ potential partitions and ${6 \brace 2 } = 2 ^ 5 - 1 = 31$ dichotomic partitions. In particular, we were interested in evaluating the performance of the method in terms of sensitivity and specificity. Sensitivity is defined as the ratio of positive cases in the ground truth model (corresponding here to existing patterns of the form $X_a \nindep X_{\bar{a}}$) that are actually detected as positive. As to specificity, it is the ratio of negative cases in the ground truth model (corresponding to existing patterns of the form $X_a \indep X_{\bar{a}}$) that are actually detected as negative.

\subsection{Data}

For $n = 6$, we considered partitions with an increasing number of blocks $K$ ($1 \leq K \leq 6$). For a given value of $K$, we performed 500 simulations, for a total of 3\,000 simulations. For each simulation, the 6 variables were randomly partitioned into $K$ clusters, all partitions having equal probability of occurrence (\gcitep[Chap.~12]{Nijenhuis-1978}; \gcitep{Wilf-1999}). For a given partition $a_1 \mid \dots \mid a_K$ of $[6]$, we generated 300 i.i.d. samples following either an univariate (if the size $n_k$ of $a_k$ was equal to 1) or a multivariate (if $n_k > 1$) normal distribution with mean $0$ and covariance matrix $\Sigma_k$ sampled according to a Wishart distribution with $n_k + 1$ degrees of freedom and scale matrix the identity matrix and then rescaled to a correlation matrix. Such a sampling scheme on $\Sigma_k$ generated correlation matrices with uniform marginal distributions for all correlation coefficients \gcite{Barnard_J-2000}.

\subsection{Analysis}

For each of the 3\,000 simulations, we considered subsets of size varying from 50 to 300 by increment of 50. For each dataset, we computed the $p$-values of the minimum discrimination information statistics corresponding to the 31 relationships of dichotomic independence under the null hypothesis that they are equal to 0 (Section \ref{ss:dpi:inf}) using the asymptotic chi-squared distribution. These $p$-values were then used to evaluate the inference quality using two approaches.
\par
First, we computed the area under curve (AUC) of  the receiver operating characteristic (ROC) curves corresponding to all datasets \gcite{Fawcett-2006}. AUC is a classical way to assess the performance of a binary classifier. More precisely, we applied significance levels of increasing values in $[0, 1]$. For each level, we computed the rates of false positives (i.e., the number of patterns $X_a \indep X_{\bar{a}}$ in the model that were wrongly detected as $X_a \nindep X_{\bar{a}}$ in the data divided by the total number of patterns $X_a \indep X_{\bar{a}}$ in the model) and the rate of true positives (i.e., the number of patterns $X_a \nindep X_{\bar{a}}$ in the model that were correctly detected as such in the data divided by the total number of patterns $X_a \nindep X_{\bar{a}}$ in the model). Plotting the true positive rate (sensitivity) as a function of the false positive rate (one minus specificity) yielded a ROC curve, whose area under the curve yielded the AUC. AUC ranges between 0 and 1, with perfect separation power  corresponding to 1, while the AUC corresponding to a random inference procedure is expected to be around 0.5. Since sensitivity could not be computed in the case of a 6-block model (all patterns of dichotomic independence hold, so there is no true positive in the model), and similarly for specificity in the case of a 1-block model (no pattern of dichotomic independence holds, so there is no true negative in the model), AUC was only obtained for data generated from models with 2, 3, 4, or 5 blocks.
\par
As a second method of assessment, we computed the sensitivity and specificity corresponding to a fixed significance level of $\alpha = 0.1$ with FDR-controlling procedure. We also computed the ratio of finest patterns of mutual independence that were correctly retrieved at that significance level.

\subsection{Results}

Results are summarized in Figure~\ref{fig:sim:res:1}. Globally, performance improved both in average and variability with increasing data size. AUC was found to often be close to 1, with variability increasing with the number of underlying blocks in the simulation model. This result shows that, given a correct significance level, the inference procedure could separate existing from non-existing patterns of dichotomic independence with very high accuracy and, therefore, infer the correct pattern of mutual independence.

\begin{figure}[!htbp]
 \centering
 \begin{tabular}{cc}
 (a) & (b) \\
  \includegraphics[width=0.45\columnwidth]{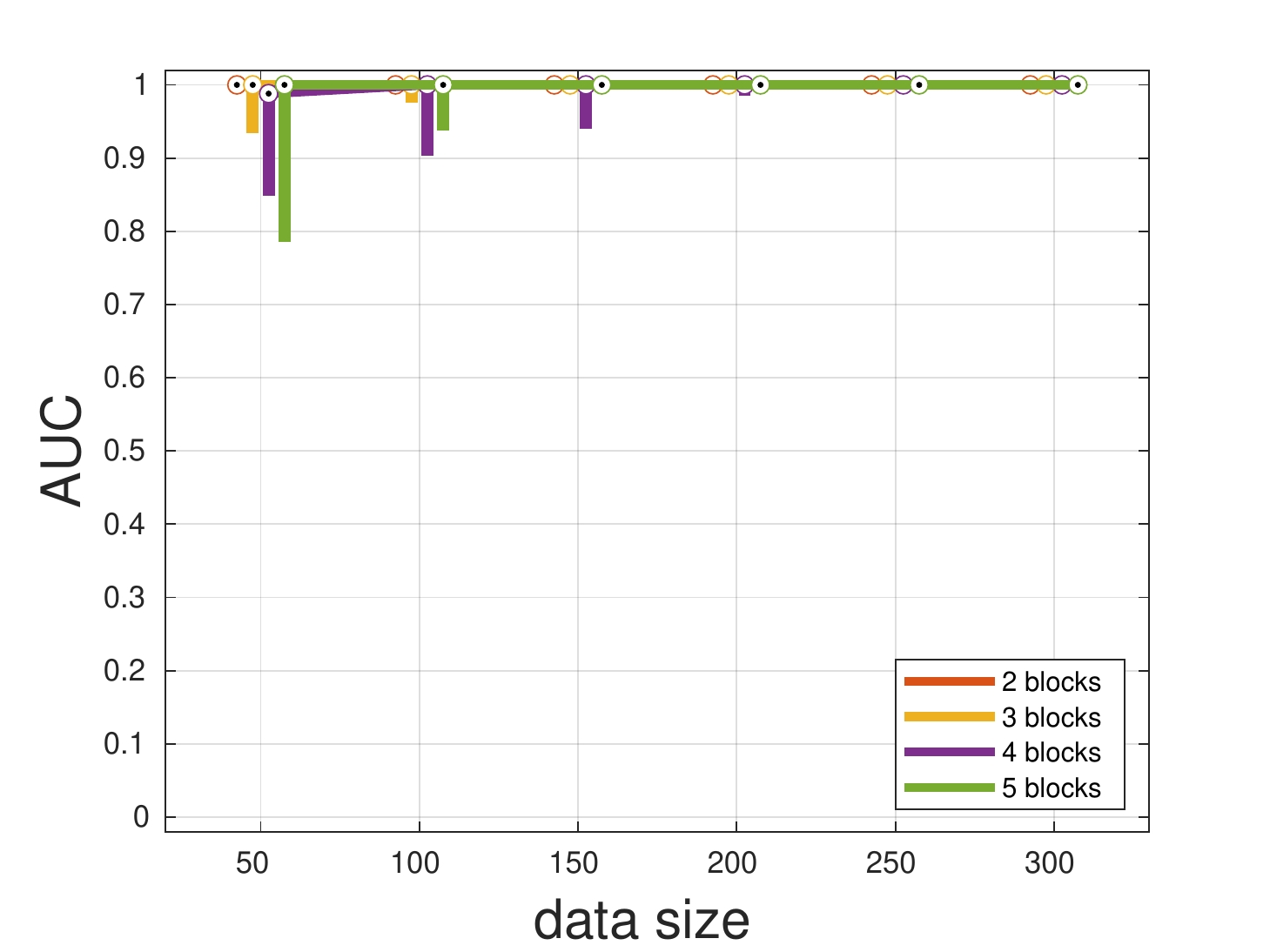} &
  \includegraphics[width=0.45\columnwidth]{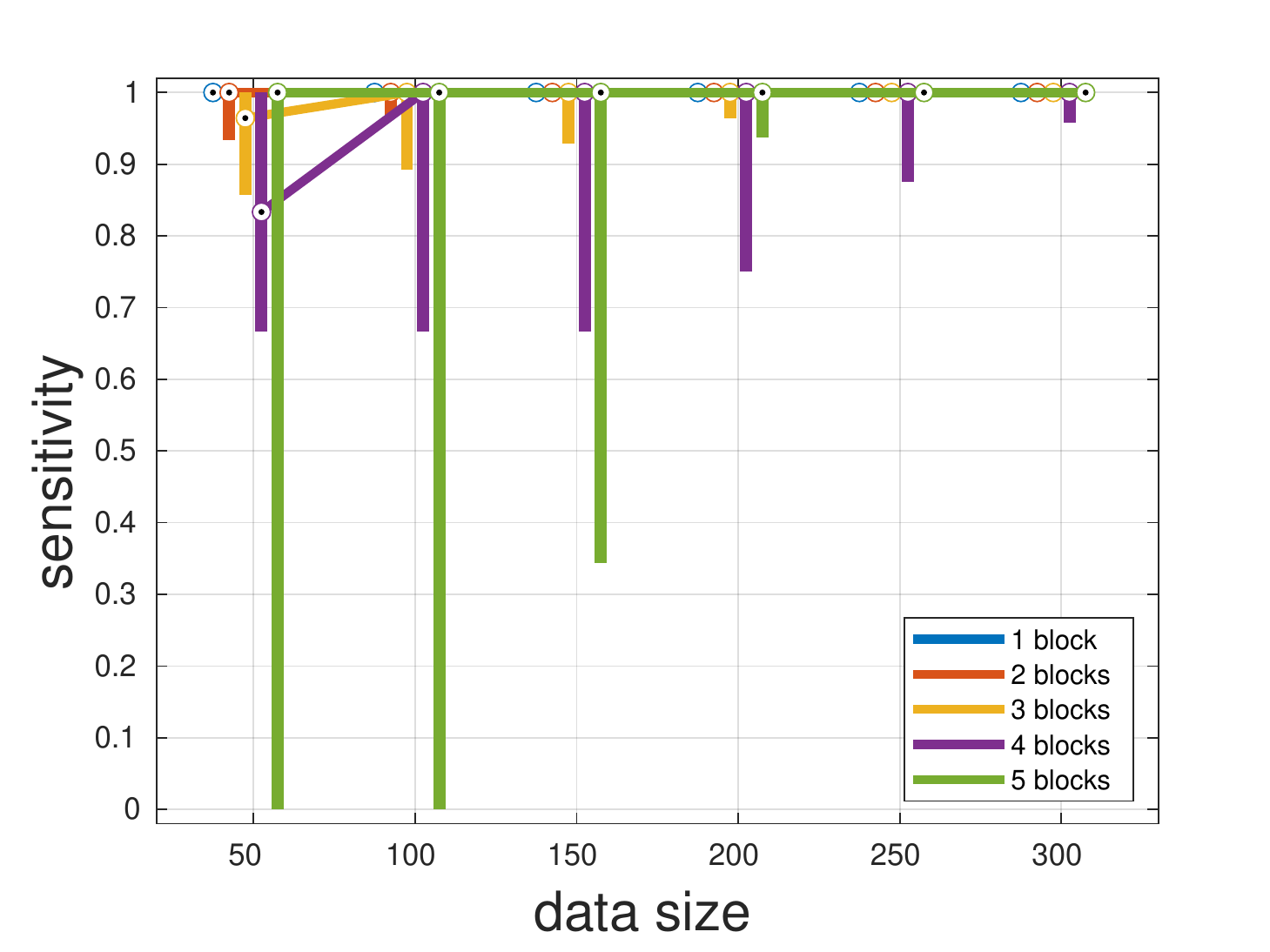}
  \\
  (c) & (d) \\
  \includegraphics[width=0.45\columnwidth]{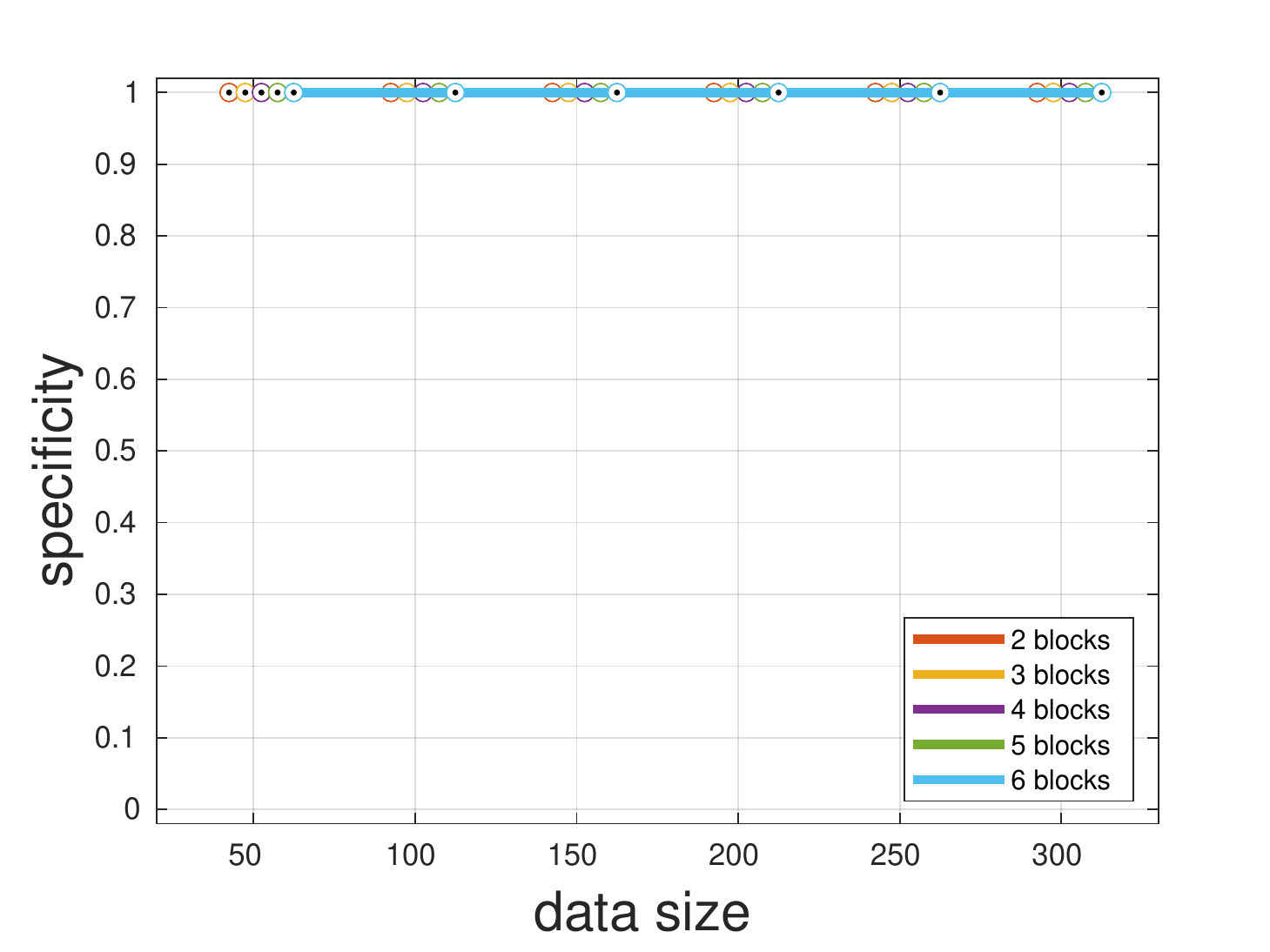} &
  \includegraphics[width=0.45\columnwidth]{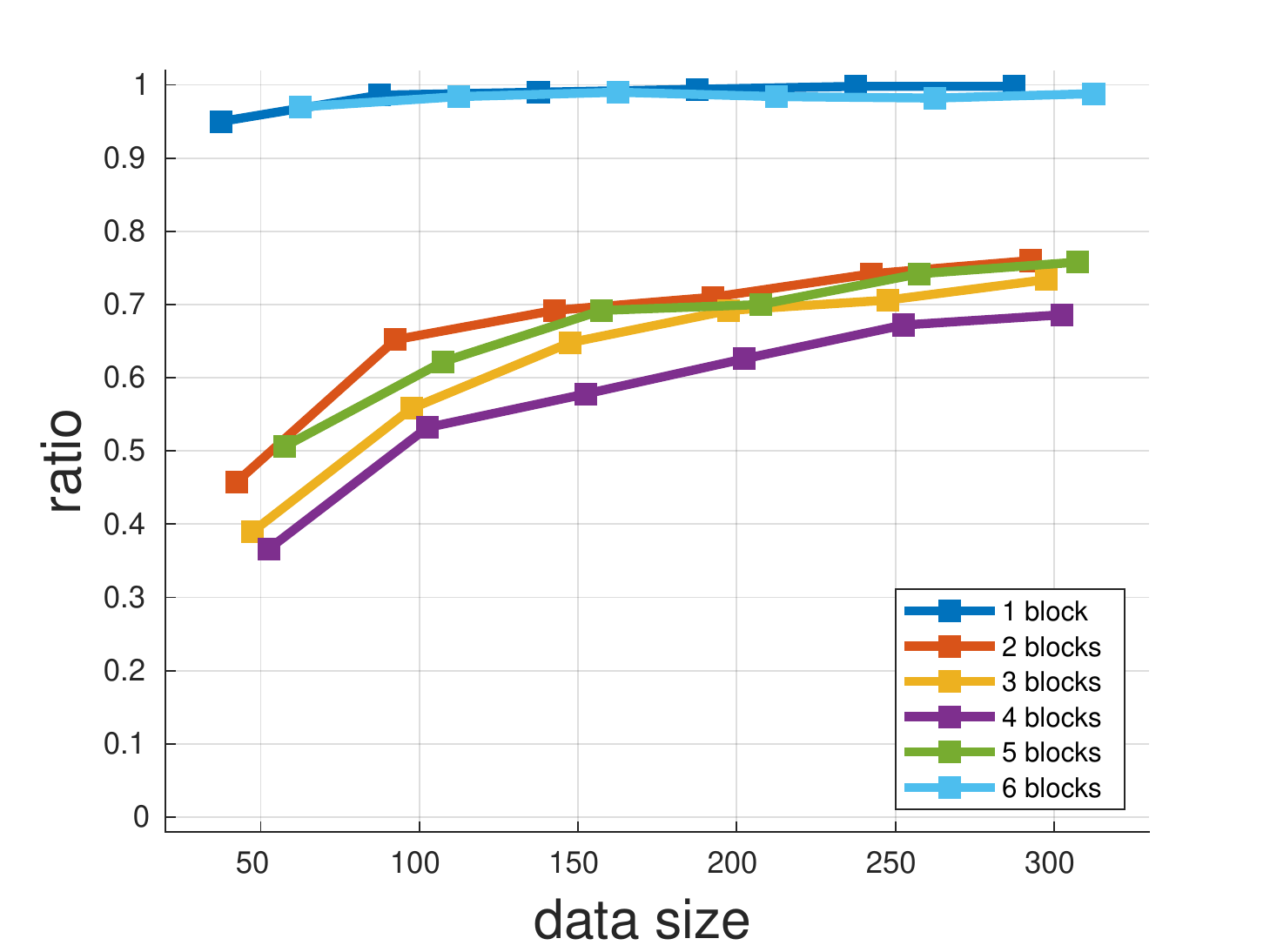}
 \end{tabular}
 \caption{\textbf{Simulation study.} (a) AUC. For a significance level $\alpha = 0.1$: (b) Sensitivity; (c) Specificity. (d) Ratio of patterns of mutual independence correctly detected. (a), (b) and (c) are boxplots (median and $[25\%,75\%]$ frequency interval).}
 \label{fig:sim:res:1}
\end{figure}

To further investigate AUC variability, we plotted AUC as a function of the absolute average correlation in blocks (see Figure~\ref{fig:sim:res:2}). As expected, the results show that the inference procedure was adversely affected by low within-block correlation levels and performed better when within-block correlation was larger on average. 

\begin{figure}[!htbp]
 \centering
 \begin{tabular}{cc}
  2 blocks & 3 blocks \\
  \includegraphics[width=0.45\columnwidth]{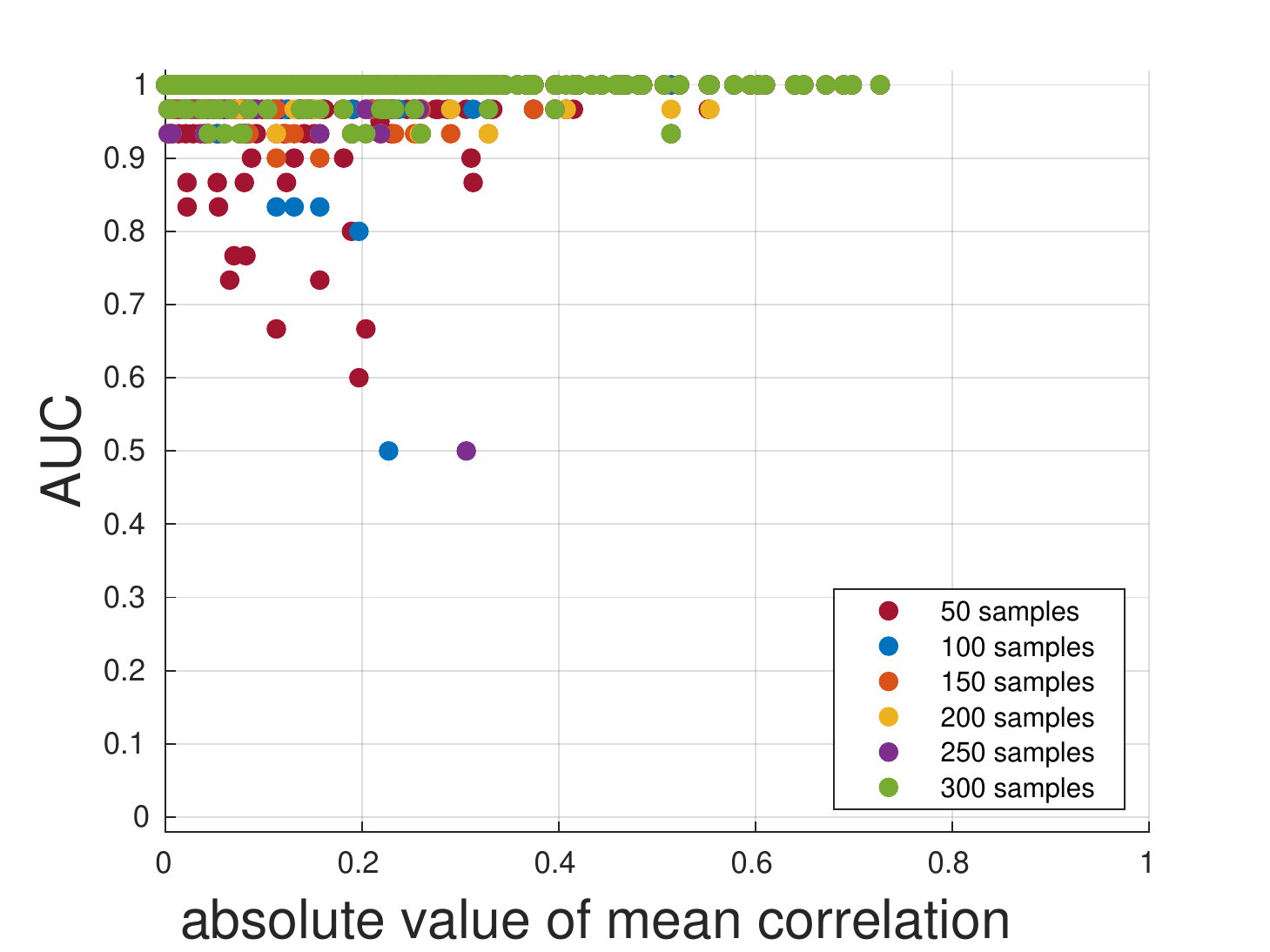} &
  \includegraphics[width=0.45\columnwidth]{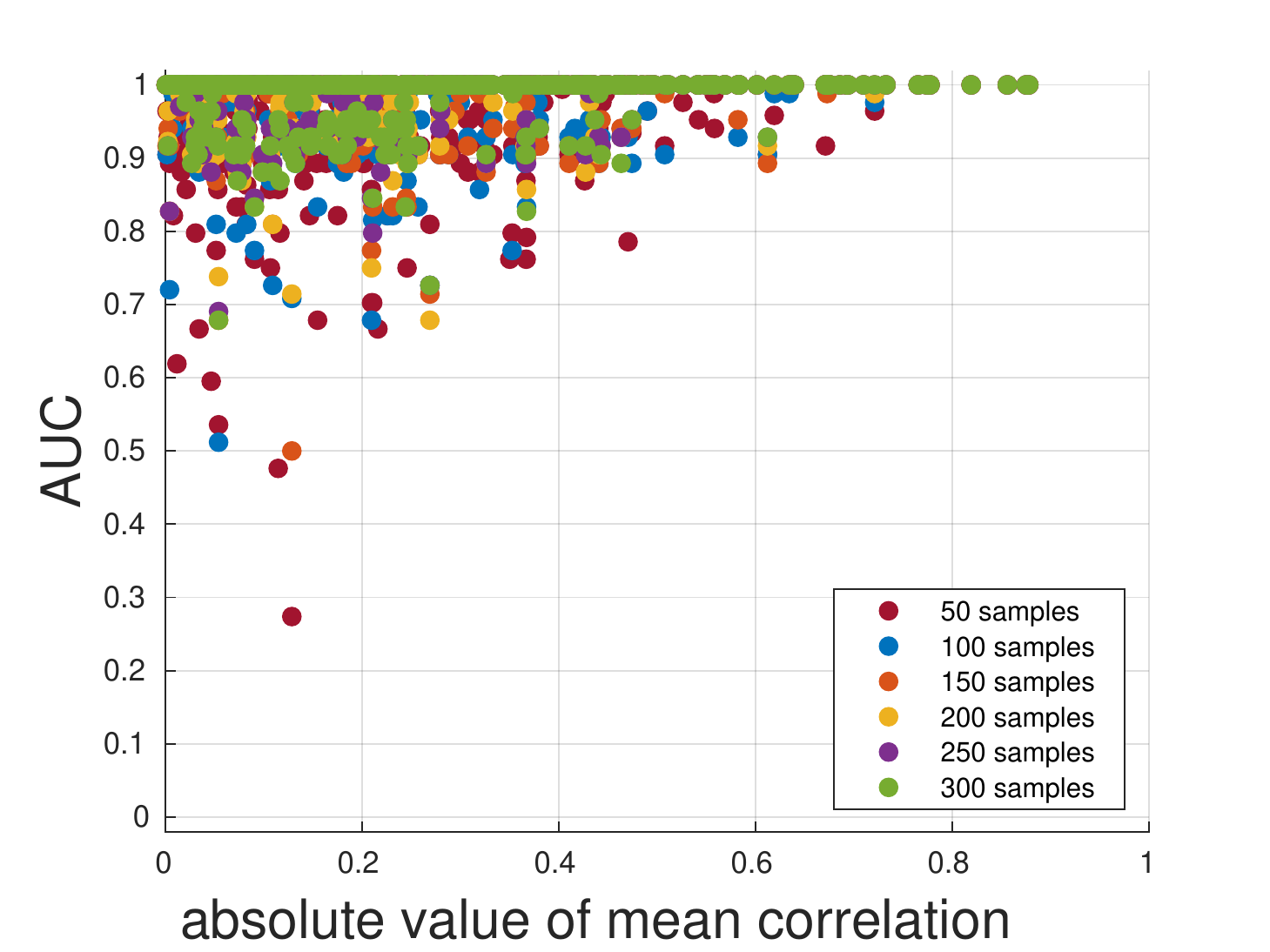}
  \\
  \\
  4 blocks & 5 blocks \\
  \includegraphics[width=0.45\columnwidth]{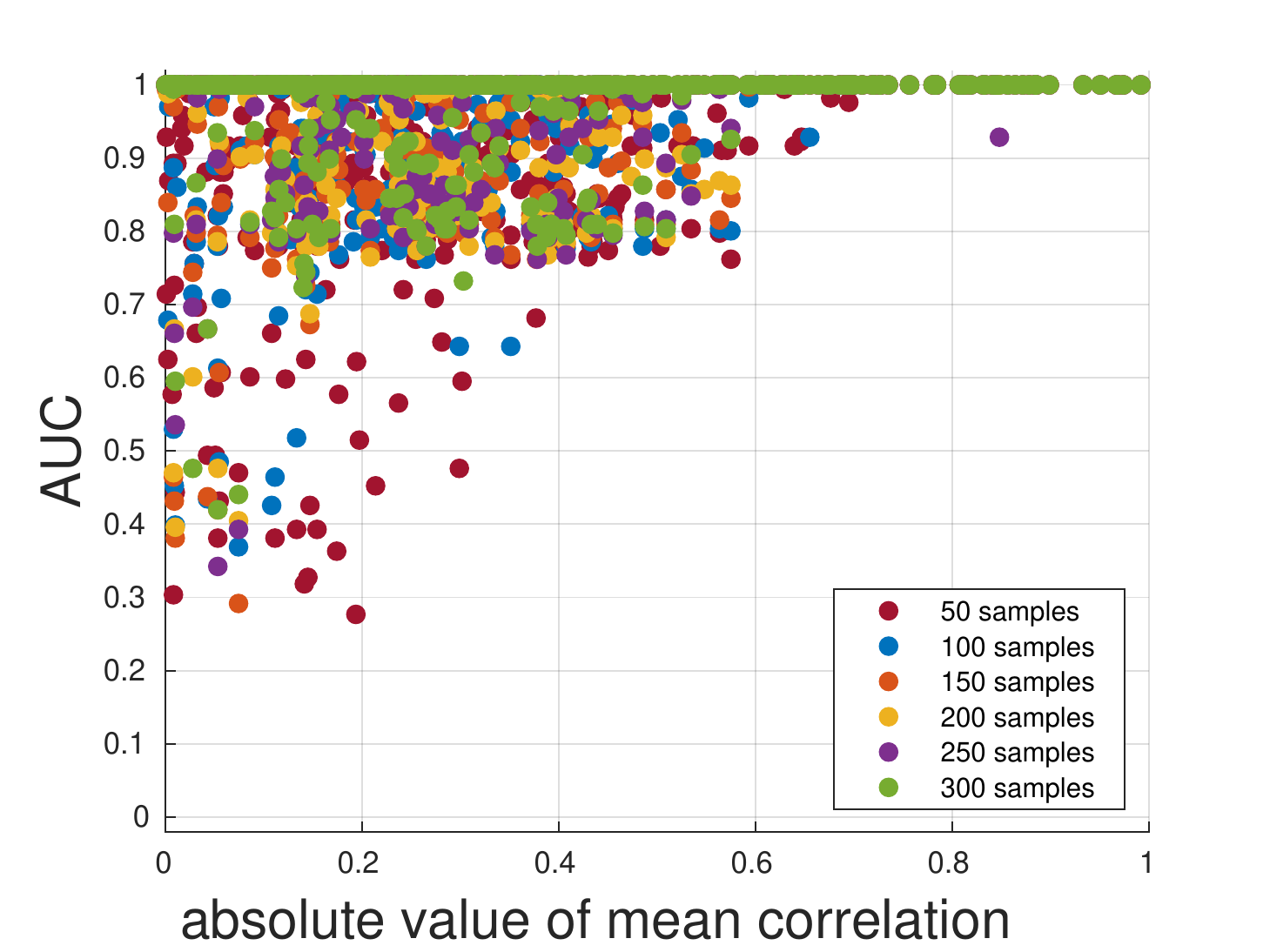} &
  \includegraphics[width=0.45\columnwidth]{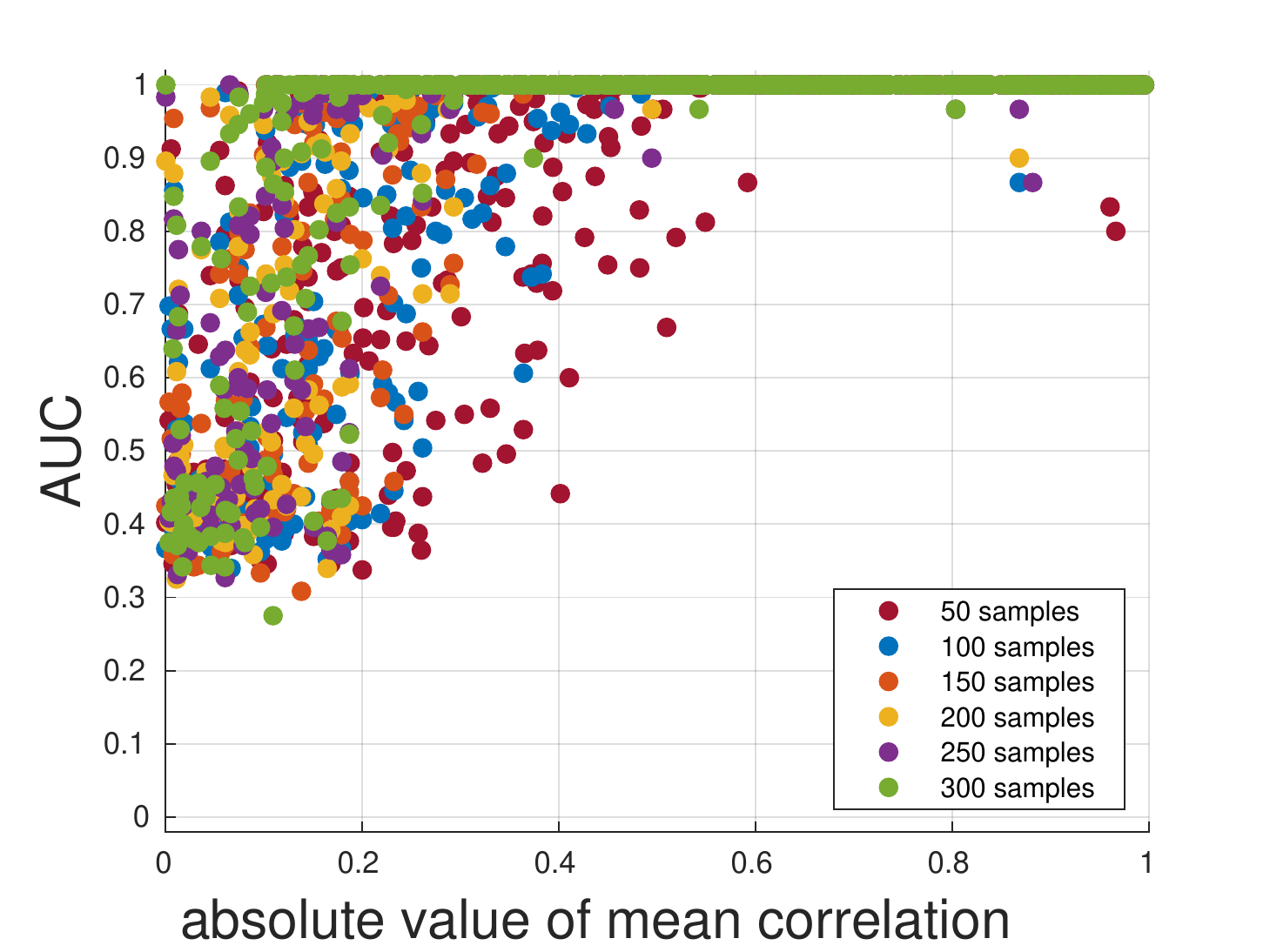}
 \end{tabular}
 \caption{\textbf{Simulation study.} Value of AUC as a function of the absolute value of the average correlation within blocks.}
 \label{fig:sim:res:2}
\end{figure}

With a fixed significance level of 0.1, specificity was very close to 1 for all data sizes and numbers of blocks, giving evidence in favor of an excellent detection of existing patterns of dichotomic independence. By contrast, sensitivity appeared to be poorer, with a level that tended to decrease with an increasing number of blocks and a variability that tended to increase with an increasing number of blocks. In other words, it was harder for the method to correctly extract patterns of the form $X_a \nindep X_{\bar{a}}$, and we tended to extract too many patterns of dichotomic independence.
\par
As a reference, we also analyzed the simulated data using a simple thresholding procedure with FDR. The results are summarized in Section~2 of \supplmat. In terms of ratio of correctly inferred patterns, results were similar for a 1-block partition (i.e., no pattern of mutual independence), and about 10\% worse for the other cases; for partitions with 2 to 5 partitions, such a difference was particularly observed for larger data sizes, while the result was mostly independent of data size for 6-block partitions (i..e, totally independent variables).

\section{Toy example} \label{s:hiv}
 
In this section, we considered a simple investigation of mutual independence patterns in real data \gcite{Roverato-1999, Marrelec-2006b, Marrelec-2015, Marrelec-2021b}. Akin to the simulation study, it involves 6 variables.

\subsection{Data}

The data originates from a study investigating early diagnosis of HIV infection in children from HIV positive mothers \gcite{Roverato-1999}. The variables are related to various measures on blood and its components: $X_1$ and $X_2$ immunoglobin G and A,
respectively; $X_4$ the platelet count; $X_3$ , $X_5$ lymphocyte B
and T4, respectively; and $X_6$ the T4/T8 lymphocyte ratio. The observed correlation matrix is given in Table~\ref{tab:rov:corr}. Experts expected the existence of a strong association between variables $X 1$ and $X_2$ as well as between variables $X_3$, $X_5$, and $X_6$. 
The data was analyzed using conditional independence graphs, suggesting no connections between $X_4$ and other variables \gcite{Roverato-1999, Marrelec-2006b}. This assumption was confirmed when investigating mutual independence patterns \gcite{Marrelec-2015, Marrelec-2021b}. Our question here is: Is $1 2 3 5 6 \mid 4$ the finest pattern of mutual independence?
 
\begin{table}[!htbp]
 \caption{\textbf{HIV study data.} Summary statistics for the HIV data. Sample variances (main diagonal, bold), correlations (lower triangle) and partial correlations (upper triangle, italic). Data from \gcitet{Roverato-1999}.}
\label{tab:rov:corr}
 \centering
 \begin{tabular}{c|cccccc}
        & $X_1$ & $X_2$ & $X_3$ & $X_4$ & $X_5$ & $X_6$ \\
  \hline
  $X_1$ & $\mathbf{8.84}$ & $\mathit{0.479}$ & $\mathit{-0.043}$ & $\mathit{-0.033}$ & $\mathit{0.356}$ & $\mathit{-0.236}$ \\
  $X_2$ & 0.483   & $\mathbf{0.192}$ & $\mathit{0.068}$ & $\mathit{-0.084}$ & $\mathit{-0.224}$ & $\mathit{-0.110}$ \\
  $X_3$ & 0.220 & 0.057 & $\mathbf{8.92 \times 10^6}$ & $\mathit{0.085}$ & $\mathit{0.552}$ & $\mathit{-0.330}$ \\
  $X_4$ & $-0.040$ & $-0.133$ & 0.149 & $\mathbf{2.03 \times 10^4}$ & $\mathit{0.091}$ & $\mathit{0.013}$ \\
  $X_5$ & 0.253    & $-0.124$ & 0.523 & 0.179 & $\mathbf{1.95 \times 10^6}$ & $\mathit{0.384}$ \\
  $X_6$ & $-0.276$ & $-0.314$ & $-0.183$ & 0.064 & 0.213 & $\mathbf{1.39}$
 \end{tabular}
\end{table}

\subsection{Analysis}

For $n = 6$ variables, there was a total of $m = 2^5 - 1 = 31$ patterns of dichotomic independence (to be compared with a total of 203 potential patterns of mutual independence). We computed the $p$-values associated to all 31 patterns using the asymptotic chi-squared distribution.

\subsection{Results}

All $p$-values were found to be lower than $10^{-4}$, except for the one associated with the partition $12356 \mid 4$, which was found to be equal to 0.332. Only the pattern $X_{12356} \indep X_4$ was not rejected for a wide range of significance levels (from $4 \times 10 ^ { - 5 }$ to 0.332), so that the result $\hat{\Delta} ( X ) = \{ 12356 \mid 4 \}$ was found to be quite robust to the choice of threshold. Since only one pattern of dichotomic independence was found to hold, $\hat{\mu} ( X )$ was equal to it, $\hat{\mu} ( X ) = 12356 \mid 4$. In other words, the inferred finest pattern of mutual independence is that $X_{12356} \indep X_4$.

\section{Real data} \label{s:rd}

We also considered an application of our method to real data consisting of brain recordings induced by an electrical stimulation of the median nerve at wrist level. It is well known that such a stimulation is associated with a typical response known as a somatosensory evoked potential (SEP). Our objective here was to investigate potential dependencies between various frequencies bands of the SEP.

\subsection{Data}

Somatosensory evoked potentials following median nerve stimulations were recorded in a healthy subject. Brain responses were acquired using multichannel EEG with a sampling frequency of  3~kHz. Electrical median nerve stimulation of 1~ms duration was applied to median nerve at the wrist level. The stimulus was applied 300 times, with a 500-ms inter-trial interval. Following previous recommendations, we studied the channels recorded from peri-central sulcus (CP3--Fz). Data acquisition was performed at the Center for Neuroimaging Research (CENIR) of the Brain and Spine Institute (ICM, Paris, France). The experimental protocol was approved by the CNRS Ethics Committee and by the national ethical authorities (CPP {\^I}le-de-France, Paris 6 -- Pitié-Salpêtrière and ANSM). To avoid artefacts induced by the stimulation, we focused on a time window ranging between 10 and 100~ms after stimulation.

\subsection{Analysis}

We considered the power spectral density (PSD) of the SEP as estimated by Welch's method \gcite{Welch-1967} at $n = 10$ frequency values uniformly spaced between 40 and 1000~Hz in log scale (see Table~\ref{tab:analyse:freq}). We defined $X_i$ as the log-10 of the PSD at frequency $f_i$. The data gave us access to $k = 300$ i.i.d. realizations of $X = X_{[10]}$. $X$ can be associated with $\varpi_n = 115\,975$ different patterns of mutual independence and ${ n \brace 2 } = 511$ patterns of dichotomic independence. The $p$-values were computed using the asymptotic chi-squared distribution.

\begin{table}[!htbp]
 \centering
 \caption{\textbf{Real data.} Value of frequencies (in Hz) used for power spectrum estimation.} \label{tab:analyse:freq}
 \begin{tabular}{cccccccccc}
  $f_1$ & $f_2$ & $f_3$ & $f_4$ & $f_5$ & $f_6$ & $f_7$ & $f_8$ & $f_9$ & $f_{10}$ \\
  \hline
  40 & 57.2 & 81.8 & 117 & 167 & 239 & 342 & 489 & 699 & 1000   
 \end{tabular}
\end{table}

\subsection{Results}

Results are summarized in Figure~\ref{fig:analyse:res} and Table~\ref{tab:analyse:res}. For significance levels ranging between about 0.05 and 0.2 with FDR correction, 7 patterns of dichotomic independence were found to be nonsignificant. $\hat{\Delta} ( X )$ was therefore composed of these 7 elements, whose intersection yielded the finest pattern of mutual independence
\begin{equation}
 \hat{\mu} ( X ) = 1 2 3 4 5 6 8 \mid 7 \mid 9 \mid 10.
\end{equation}
On the one hand, the $\log_{10}$-PSDs corresponding to lower frequencies ($X_1$ to $X_6$) were grouped together as dependent. These dependencies seemed to be at least partly driven by strong correlations between neighboring frequencies. The $\log_{10}$-PSDs corresponding to higher frequencies ($X_7$, $X_9$ and $X_{10}$) were found to be independent from all other frequencies. By contrast, $X_8$ was grouped with the lower frequencies. This dependency could mostly be seen in the form of stronger correlation values with $X_1$ and $X_2$. These results provide evidence for the fact that electrical stimulation of the median nerve has a global effect on  the signal in the 10--100~ms time window, both at low frequency and at higher frequencies. Furthermore, we expect different physiological processes to be at the origin of the observed responses at higher frequency: some that are related to lower frequency processes, some that are not.

\begin{figure}[!htbp]
 \centering
 \begin{tabular}{cc}
  \includegraphics[width=0.5\linewidth]{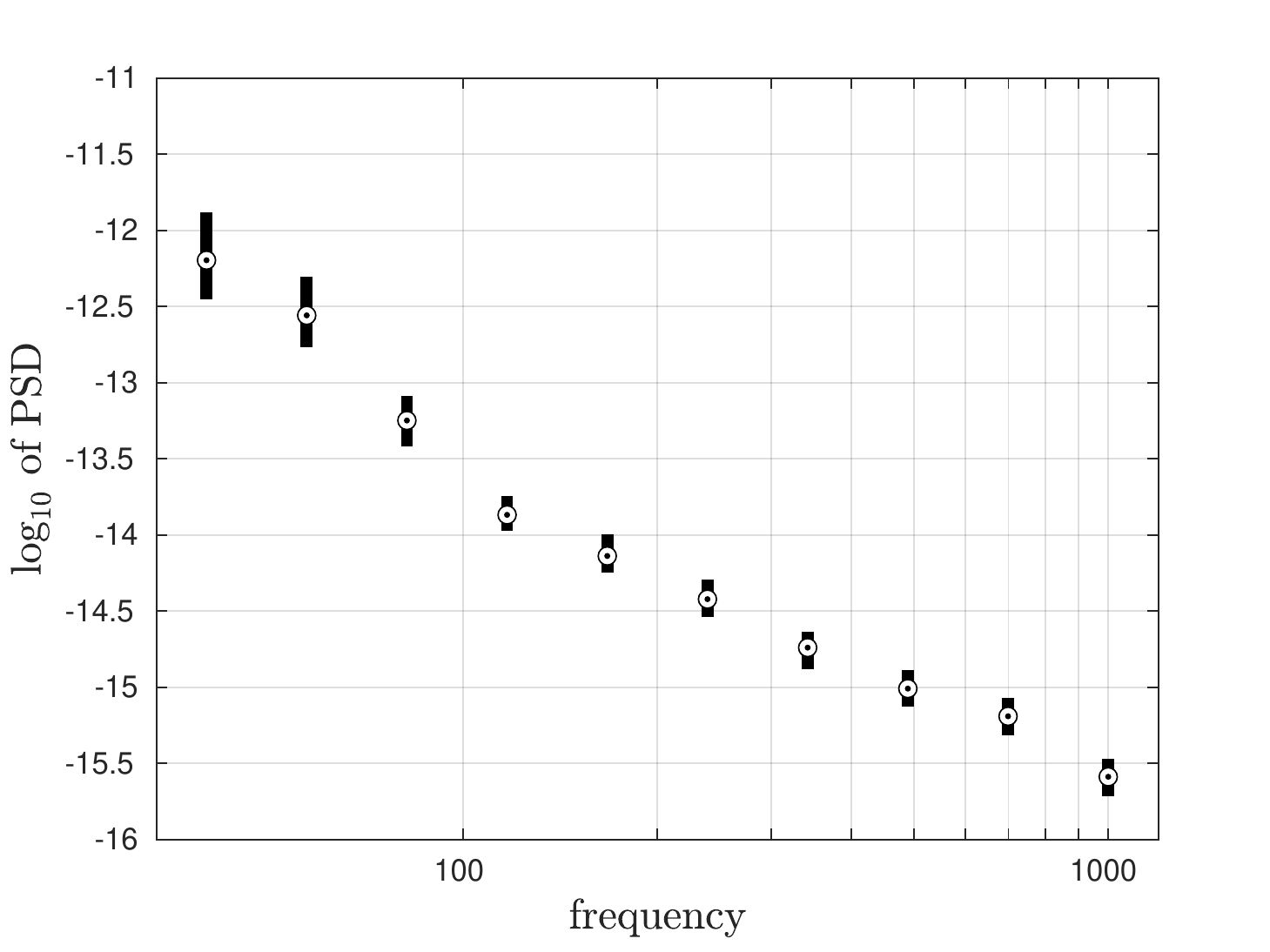}
  & \includegraphics[width=0.5\linewidth]{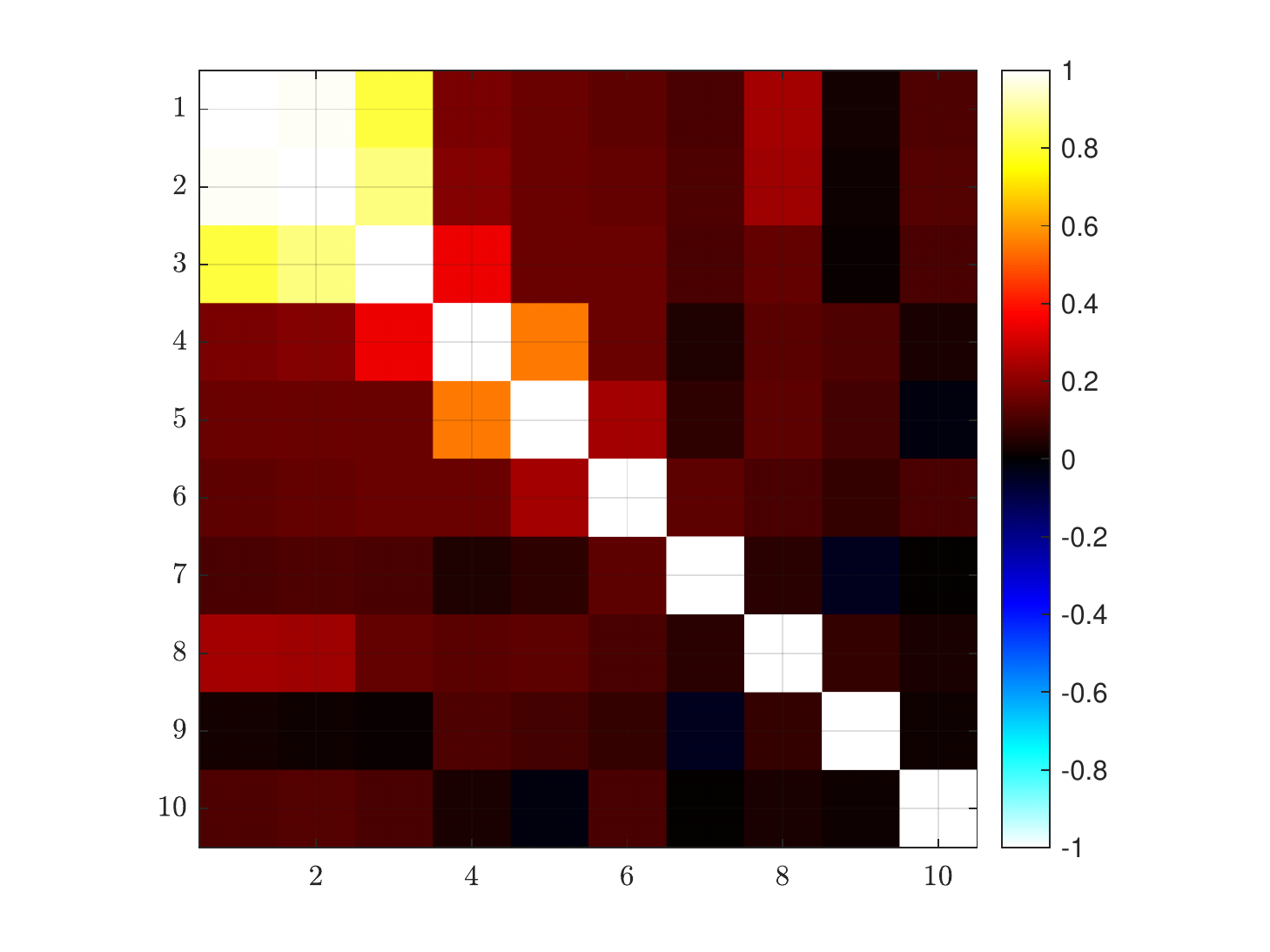} \\ 
  \includegraphics[width=0.5\linewidth]{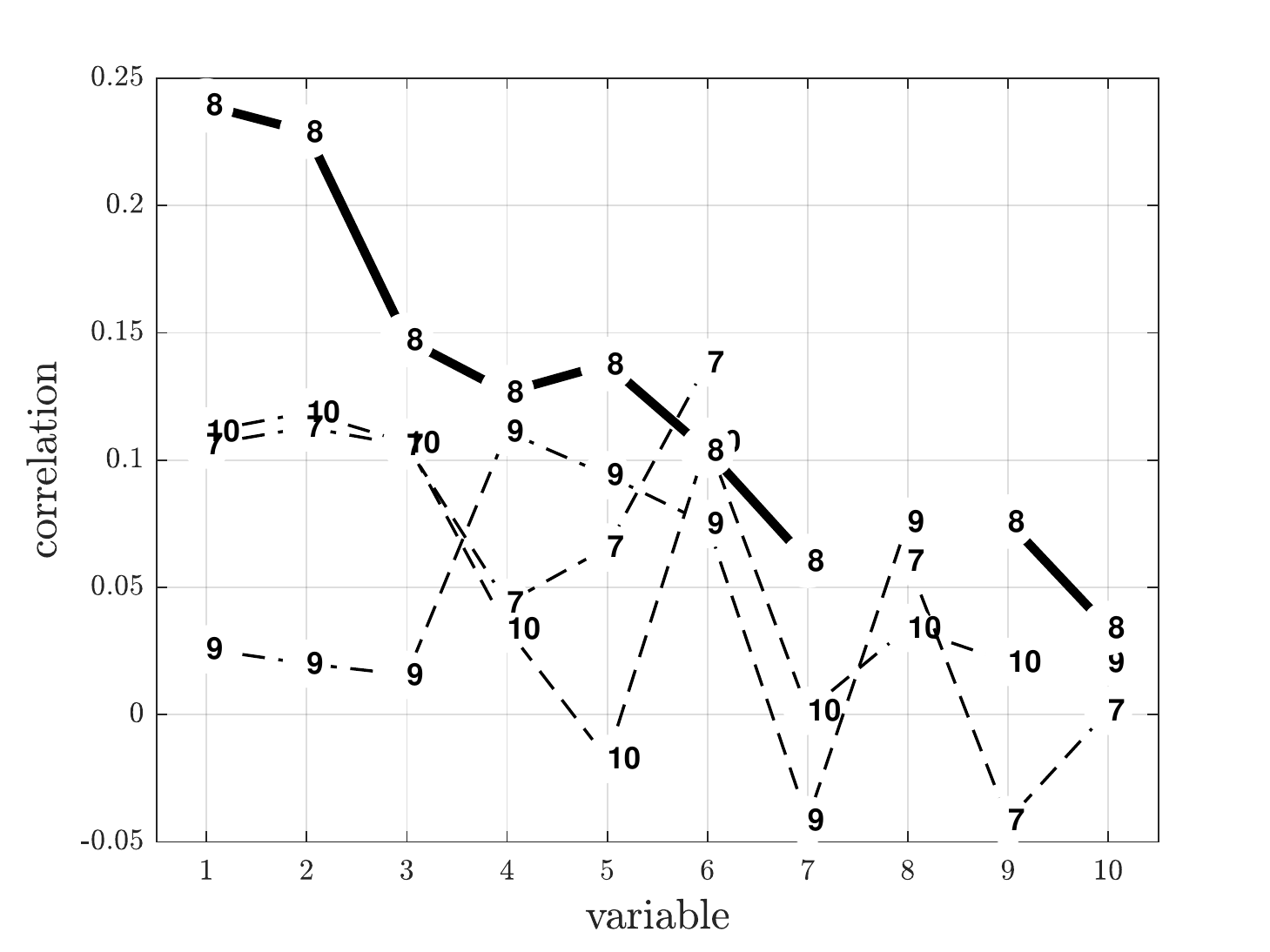}
  & 
 \end{tabular}
 \caption{\textbf{Real data.} Top left: Boxplot (median and $[25\%,75\%]$ frequency interval) of $X$ over the 300 stimulations. Top right: Sample correlation matrix of $X$. Bottom left: Sample correlation of $X_7$, \dots, $X_{10}$ with all other variables.} \label{fig:analyse:res}
\end{figure}

\begin{table}[!htbp]
 \centering
 \caption{\textbf{Real data.} 8 largest $p$-values and corresponding patterns of dichotomic independence.} \label{tab:analyse:res}
 \begin{tabular}{cc}
  Pattern of dichotomic independence & $p$-value \\
  \hline
  $X_{ 9 } \indep X_{ [10] \setminus \{ 9 \} }$ & 0.616 \\
  $X_{ \{ 7, 9 \} } \indep X_{ [10] \setminus \{ 7, 9 \} }$ & 0.469 \\
  $X_{ 7 } \indep X_{ [10] \setminus \{ 7 \} }$ & 0.366 \\
  $X_{ \{ 9, 10 \} } \indep X_{ [10] \setminus \{ 9, 10 \} }$ & 0.352 \\
  $X_{ 10 } \indep X_{ [9] }$ & 0.311 \\
  $X_{ \{ 7, 10 \} } \indep X_{ [10] \setminus \{ 7, 10 \} }$ & 0.215 \\
  $X_{ \{ 7, 9, 10 \} } \indep X_{ [10] \setminus \{ 7, 9, 10 \} }$ & 0.211 \\
  \hline
  $X_{ \{ 1, 2, 3, 4, 5, 8, 10 \} } \indep X_{ \{ 6, 7, 9 \} }$ & 0.0469
 \end{tabular}
\end{table}

\section{Discussion} \label{s:disc}

In the present manuscript, we were interested in the blind extraction of patterns of mutual independence from data and, more specifically, of one particular pattern: the \emph{finest} one. We used the connection between mutual independence and partitions together with the lattice structure of partitions. We showed that the set $\Pi ( X )$ of patterns of mutual independence that hold for a given multidimensional variable $X$ is a sublattice. This sublattice has a unique finest pattern $\mu ( X )$ of mutual independence on $X$ which is the intersection of all patterns of $\Pi ( X )$. We then introduced a specific kind of independence that we called dichotomic independence and showed that, if $\Delta ( X )$ is the set of all patterns of dichotomic independence holding for $X$, then $\mu ( X )$ is also the intersection of all patterns of $\Delta ( X )$. We finally proposed a method to estimate $\Delta ( X )$ from a dataset consisting of i.i.d. realizations of a multivariate normal distribution. The method was tested on simulated data and applied to a toy example and experimental data.
\par
Our approach strongly relies on the lattice structure of partitions, on the statistical properties of the minimum discrimination information statistic, as well as on the FDR procedure. From a mathematical perspective, the core of the method is Theorem~\ref{th:Delta}, which shows that the finest pattern of mutual independence $\mu ( X )$ can be exactly recoved as the intersection of all patterns of dichotomic independence, i.e., of all elements of $\Delta ( X )$. It is an advantage of the method that the theory relies on mathematical properties in abstract algebra, as this part is valid regardless of  the underlying data distribution and size.
\par
From a statistical perspective, the inference process allowed us to take advantage of the theoretical result and investigate the independence structure of data in the case of i.i.d. realizations of a multivariate normal variable. Since the procedure is consistent, the probability to have type~II errors (false negatives) tends to 0 as the data size tends to infinity. By contrast, type~I errors may not vanish, as the procedure tends to impose a fraction of false discoveries, i.e., false positive \gcite{Benjamini-1995, Storey-2004, Blanchard_G-2014}---unless one selects a threshold that itself tends to 0 as the data size tends to infinity \gcite{Neuvial-2012}. This is the price to pay for the fact that the FDR procedure is in general more powerful than procedures controlling the family-wise errors, as, e.g., Bonferroni procedure. When investigating mutual independence, researchers have usually already extracted a subset of variables that they think should be dependent. In this perspective, it is the potential existence of nontrivial patterns of independence that is of interest, as these are later subject to interpretation in terms of underlying mechanisms and allow to analyze independent groups of variables separately. It is therefore important to avoid overly conservative tests, as they would tend to falsely detect non-existing patterns of independence and, in subsequent analyses, separately investigate variables that are actually dependent.  Improving the FDR is a field of active research \gcite[see, e.g.,][]{Benjamini-2010, Genovese-2015}, and our method can easily be adapted to apply a wide range of multiple comparisons correction methods.
\par
A key advantage of the method is the reduction of dimensionality that it is able to perform. Indeed, one of the reasons of the complexity to extract patterns of mutual independence is that the space of potential patterns is discrete and very large; its cardinality is given by the $n$th Bell number $\varpi_n$ \gcite{Rota-1964}, which grows faster that an exponential but slower than a factorial---it is $O [ ( n / \ln n ) ^n ]$. By contrast, our approach proposes to test the subset of all dichotomic partitions. The potential number of patterns of dichotomic independence for an $n$-dimensional variable is given by the Stirling number of the second kind ${n \brace 2} = 2^{n-1} - 1$ as mentioned earlier. While this quantity grows quickly for large $n$, it is still substantially smaller than $\varpi_n$---see, e.g., Table~\ref{tab:bell-stirling} for a few comparative examples. For instance, for $n = 20$, listing and storing $\varpi_n = 5.17 \times 10^{13}$ patterns of mutual independence is out of reach of many computers, while the number of patterns of dichotomic independence ${ n \brace 2 } = 524\,287$ is large but still manageable. Note that, unlike greedy algorithms, such a reduction of the search space is not based on a heuristic, a local search, nor a stepwise procedure, as our approach provides an exact one-step, global solution thanks to Theorem~\ref{th:Delta}.

\begin{table}[!htbp]
 \centering
 \caption{\textbf{Number of potential patterns of mutual independence and dichotomic independence for an $n$-dimensional variable.}} \label{tab:bell-stirling}
 \begin{tabular}{c|ccccccccccc}
  $n$ & 1 & 2 & 3 & 4 & 5 & 6 & 7 & 8 & 9 & 10 & 20 \\
  \hline
  $\varpi_n$ & 1 & 2 & 5 & 15 & 52 & 203 & 877 & 4\,140 & 21\,147 & 115\,975 & $5.17 \times 10^{13}$ \\
   ${ n \brace 2 }$ & 0 & 1 & 3 & 7 & 15 & 31 & 63 & 127 & 255 & 511 & 524\,287
 \end{tabular}
\end{table}

While the theoretical result of Theorem~\ref{th:Delta} ensures that $\mu ( X )$ can always be obtained from $\Delta ( X )$, the efficiency of the methods relies in great part on how well the inference procedure is able to estimate $\Delta ( X )$ from data. On simulated data with $n = 6$, we showed that the method was able to perform well. With AUCs close to 1, there were many cases where an optimal threshold existed, separating negative and positive values almost perfectly, and leading to a correct retrieval of $\Delta ( X )$ and, therefore, $\mu ( X )$. When the significance threshold was set at $\alpha = 0.1$, we also found high specificity (showing that existing patterns of dichotomic independence could very often be correctly detected), but sensitivity appeared to behave more poorly. We believe that these results hint for a suboptimal choice of the significance level and suggest that there is still room for improvement in the choice of this value.
\par
We mentioned the theoretical result that the minimum discrimination statistic has a distribution that can be approximated by a noncentral chi-squared distribution and, asymptotically, a chi-squared distribution (Section~\ref{sss:test}). While we expected the noncentral distribution to provide better inference, in particular for smaller data sizes,  we actually found out that it exhibited poorer performance on the simulated data than the asymptotic chi-squared distribution (see Section~3 of \supplmat). The reason for this unexpected result remains a puzzle to us. As a consequence, we applied the asymptotic chi-squared distribution for all analyses.
\par
The main result of this work is the possibility to extract the finest pattern of mutual independence using dichotomic independence. This result was applied to a statistical framework assuming normal data. In this case, mutual independence could be tested using an asymptotically exact test based on the minimum discrimination information statistic. It is however interesting to consider what is specific to normal distributions and what is valid regardless of the distribution. Importantly, the mathematical framework leading to the main result (Sections 2, 3.1 and 3.2) is related to the lattice structure of the set of partitions and, as such, is valid regardless of the underlying distribution. Also, mutual information, from which the minimum discrimination information statistic is derived, is a measure that is always positive, and is equal to 0 if and only if the two variables are independent, regardless of the underlying distribution. By contrast, what is specific of normal variables is (i) the fact that the covariance matrix fully determines all patterns of independence, (ii) the expression of mutual independence as a function of the covariance matrix, Equation~\eqref{eq:mdis:def}, and (iii) the distribution of the test under the null hypothesis of independence between the two blocks of variables. To adapt the method to non-Gaussian variables, one would have to (1) either find an estimator for mutual information adapted to the situation or use another measure, and (2) determine the distribution of this measure under the null hypothesis. For (1), one could think of nonparametric estimators of mutual information \gcite[e.g.,][]{Kraskov-2005b}, with theoretical properties that still remain to be investigated (at least in part), or measures already proposed in order to investigate independence between two variables and referred to in the introduction (e.g., in the multivariate case, \gcitep{Jupp-1980}; \gcitep[Chaps. 2 and 8]{Cover_TM-1991}; \gcitep{Bakirov-2006}; \gcitep{Schott-2008}; \gcitep{Jiang_D-2012}; \gcitep{Szekely-2013b}). An advantage of our approach is precisely that it relies on independence between two sets of variables, a research field that has already been investigated.  Recent advances in that field, both in terms of potential measures that coud be used and associated statistical tests, could therefore be combined to our approach to deal with the non-Gaussian case.
\par
Still, the case of normal data that we studied has the interest of showing that blind extraction of the finest pattern of mutual independence remains a challenge even when the underlying statistical model is simple and correct and the test perfectly adapted. First, it has to be kept in mind that the number of tests grows quickly as the number of variables $n$ tends to infinity, potentially narrowing the validity domain of the asymptotic results regarding consistency and the approximation of the null hypothesis distribution. We suspect that what also makes this problem especially hard is that weak dependencies are difficult to detect and are often overlooked as independence. While this is clearly a limit of our statistical investigation, it mirrors a widely accepted principle in science, where phenomena are usually first considered independent until enough evidence for dependence is gathered.
\par
For application to real data, our method has the key advantages of being theoretically principled and its application a simple and fast one-step procedure. To our knowledge, it is the first time that it is possible to blindly and noniteratively extract the finest pattern of mutual independence from real data. For example, the analysis we presented on brain recordings dealt with 10 variables, corresponding to 115\,975 potential patterns of mutual independence. Working with dichotomic independence reduced the number of tests to 511 and provided a principled way to bring the results together. Also, unlike a black box, each test of dichotomic independence gives relevant information regarding the underlying structure of independence.
\par
More broadly, we advocate that understanding the structure of mutual independence is a key issue to deal properly with independence. The present work on dichotomic independence provides a first step in this direction.

\section{Acknowledgments}

The authors would like to thank the Center for Neuroimaging Research (CENIR) of the Brain and Spine Institute (ICM, Paris, France) for the acquisition of the real data, and Véronique Marchand-Pauvert for providing them with the data. They are also grateful to Étienne Roquain for insightful discussions regarding multiple comparisons correction and FDR.

\appendix

\section{Proof of Proposition~\ref{prop:im}} \label{an:pr:prop:im}

If $X_{ a_1 }$, \dots, $X_{ a_ k }$ are mutually independent, then $\Pr ( X )$ can be decomposed as
\begin{equation} \label{eq:def:im:gen}
 \Pr (X ) = \prod_{ i = 1 } ^ k \Pr ( X_{ a_i } ).
\end{equation}
For each $i = 1, \dots, k$, let $j_i$ be the (possibly empty) element of $[l]$ such that $a_i \cap b_{ j_i } \neq \emptyset$ and let $a_i' = a_i \cap b_{ j_i }$ as well as $a_i'' = a_i \setminus b_{ j_i }$. Then $a_i' \mid a_i''$ is a bipartition of $a_i$ and each $b_j$ can be expressed as the partition of a certain number of $a_i'$'s. The joint probability of $X$ can be expressed as
\begin{eqnarray*}
 \Pr ( X ) & = & \Pr ( X_{a_1}, \dots, X_{a_k} ) \\
 & = & \prod_{ i = 1 } ^ k \Pr ( X_{ a_i } ) \\
 & = & \prod_{ i = 1 } ^ k \Pr ( X_{ a_i' }, X_{ a_i'' } ).
\end{eqnarray*}
The joint distribution of $X_{b_1}, \dots, X_{b_l}$ can be obtained by marginalization of $\Pr (X )$ with respect to the $a_i''$'s:
\begin{eqnarray*}
 \Pr ( X_{ b_1 }, \dots, X_{ b_l } ) & = & \sum_{ X_{ \cup_{ i = 1 } ^ k a_i''} } \Pr ( X ) \\
 & = & \sum_{ X_{ \cup_{ i = 1 } ^ k a_i''} } \prod_{ i = 1 } ^ k \Pr ( X_{ a_i' }, X_{ a_i'' } ) \\
 & = & \prod_{ i = 1 } ^ k  \sum_{ X_{ a_i''} } \Pr ( X_{ a_i' }, X_{ a_i'' } ) \\
 & = & \prod_{ i = 1 } ^ k \Pr ( X_{ a_i' } ) \\
 & = & \prod_{ j = 1 } ^ l \Pr ( X_{ b_j } ).
\end{eqnarray*}
This is the definition of the fact that $X_{ b_1 }$, \dots, $X_{ b_l }$ are mutually independent.

\section{Proof of Theorem~\ref{th:Pi}} \label{an:th:Pi}

To prove that $\Pi ( X )$ is a sublattice of $\Omega ( N )$, we need to prove that it is stable by join and meet. Consider $\pi_1$ and $\pi_2$ in $\Pi ( X )$.
\par
Set first $\pi_3 = \pi_1 \vee \pi_2$. Since $\pi_1 \leqslant \pi_1 \vee \pi_2$, Proposition~\ref{prop:im:pfq} entails $\pi_1 \vee \pi_2 \in \Pi ( X )$.
\par
We now set $\pi_3 = \pi_1 \wedge \pi_2$. Assume that $\pi_1 = a_1 \mid \dots \mid a_k$, $\pi_2 = b_1 \mid \dots \mid b_l$ and $\pi_3 = c_1 \mid \dots \mid c_m$. Since $\pi_3 = \pi_1 \wedge \pi_2$, each $c_i$ is of the form $a_{ \phi ( i ) } \cap b_{ \psi ( i ) }$. Since $\pi_1 \in \Pi ( X )$, we have
\begin{equation} \label{eq:prv:im1}
 \Pr ( X ) = \prod_{ i = 1 } ^ k \Pr ( X_{ a_i } ).
\end{equation}
Since $\pi_2$ is a partition of $N$, it contains in particular a covering of each $a_i$. Setting $\xi ( i )$ the subset of $[l]$ for which $a_i \cap b_{ \xi ( i ) } \neq \emptyset$, we have
$$a_i = \cup_{ j \in \xi ( i ) } ( a_i \cap b_j ).$$
Since $\pi_2 \in \Pi ( X )$, we also have
\begin{equation} \label{eq:prv:im2}
 \Pr ( X ) = \prod_{ j = 1 } ^ l \Pr ( X_{ b_j } ).
\end{equation}
Marginalization over $X_{ b_{ j \not \in \xi ( i ) } }$ leads to
\begin{equation} \label{eq:prv:im2}
 \Pr ( X_{ b_{ j \in \xi ( i ) } }  ) = \prod_{ j \in \xi ( i ) } \Pr ( X_{ b_j } ).
\end{equation}
Setting $b_j' = b_j \setminus a_i$, we marginalize with respect to the $X_{ b_j' }$, yielding
\begin{equation} \label{eq:prv:im3}
 \Pr ( X_{ a_i }  ) = \prod_{ j \in \xi ( i ) } \Pr ( X_{ a_i \cap b_j } ).
\end{equation}
Incorporating these results into Equation~(\ref{eq:prv:im1}), we obtain
\begin{eqnarray*}
 \Pr ( X ) & = & \prod_{ i = 1 } ^ k \prod_{ j \in \xi ( i ) } \Pr ( X_{ a_i \cap b_j } ) \\
 & = & \prod_{ i = 1 } ^ l \Pr ( X_{ a_{ \phi ( i ) } \cap b_{ \psi ( i ) } } ) \\
 & = & \prod_{i = 1 } ^ m \Pr ( X_{ c_i } ).
\end{eqnarray*}
We there have $X_{ c_1 }, \dots, X_{ c_m }$ mutually independent, so that $\pi_3 \in \Pi ( X )$.
\par
Since $\Pi ( X )$ is stable by the join and meet, it is a sublattice of $\Omega ( N )$. The existence and unicity of $\mu ( X )$, defined through Equation~(\ref{eq:def:mu}), is assured by the fact that $\Omega ( N )$ is a lattice. Since $\Pi ( X )$ is a sublattice, $\mu ( X ) \in \Pi ( X )$. Since it is the meet of all elements in $\Pi ( X )$, it is finer than all elements in $\Pi ( X )$, and is therefore the bottom of $\Pi ( X )$. The fact that the trivial partition is the top of $\Pi ( X )$ is obvious.

\section{Proof of Theorem~\ref{th:Delta}} \label{an:th:Delta}

Since $\mu ( X )$ is the finest pattern of mutual independence on $X$ and any $\delta \in \Delta ( X )$ is a pattern of mutual independence on $X$, we have $\mu ( X ) \leqslant \delta$. Since this is true for any $\delta \in \Delta ( X )$, this entails that $\mu ( X ) \leqslant \wedge_{ \delta \in \Delta ( X ) } \delta$. 
\par
Express $\rho = \wedge_{ \delta \in \Delta ( X ) } \delta$ as $a_1 \mid \dots \mid a_k$. Assume now that $\mu ( X ) <  \wedge_{ \delta \in \Delta ( X ) } \delta$. Then there exists $\omega$ such that $\mu ( X ) \leqslant \omega  <  \rho$ and such that a block $a_i$ of $\rho$ is decomposed into two blocks $a_{i1}, a_{i2}$ in $\omega$. This entails that
$$\mu ( X ) \leqslant \omega \leqslant a_1 \dots a_{ i - 1 } a_{ i 1 } \mid a_{ i 2 } a_{ i + 1 } \dots a_k,$$
and, as a consequence, $a_1 \dots a_{ i - 1 } a_{ i 1 } \mid a_{ i 2 } a_{ i + 1 } \dots a_k$ belongs to $\Delta ( X )$. Since $\rho$ is the meet of all elements of $\Delta ( X )$, $a_{ i 1 }$ and $a_{ i 2 }$ must belong to two different blocks of $\rho$. This is in contradiction with the fact that $a_i$ is a block of $\rho$. As a consequence, we cannot have $\mu ( X ) <  \wedge_{ \delta \in \Delta ( X ) } \delta$, i.e., we must have $\mu ( X ) =  \wedge_{ \delta \in \Delta ( X ) } \delta$. 

\section{Proof of consistency} \label{an:consist}

Consistency of a test is defined as the fact that its power (i.e., one minus the probability for type~II errors) tends to 0 as the data size $k$ tends to infinity \gcite[Chap.~2, Section~3.9]{Fraser_DAS-1957}.
\par
We first notice that the FDR procedure has more power than the family-wise error correction using Bonferroni procedure (BP). This a direct consequence of the fact that hypotheses that are rejected by BP at threshold $\alpha$ have $p$-values lower than $\alpha/m$, where $m$ is the number of tests. As a consequence, these hypotheses will also be rejected when controlling the FDR, since the corresponding $p$-values, once ordered, will all be smaller than $\alpha i / m$ for any $1 \leq i \leq m$. Defining $\Delta' ( X )$ as the complement of $\Delta ( X )$ in $\Omega_2 ( N )$, i.e., the set of patterns of dichotomic independence that do \emph{not} hold for $X$, we can therefore express the power as
\begin{equation} \label{eq:FDR-BP}
 \Pr \left( \forall \delta \in \Delta' ( X ): \mbox{$H_{ \delta }$ FDR-rejected} \right) \geq \Pr \left( \forall \delta \in \Delta' ( X ): \mbox{$H_{ \delta }$ BP-rejected} \right).
\end{equation}
We then consider a modification of our procedure, where FDR has been replaced with BP. In that case, the probability to obtain a type II error is given by
\begin{equation} \label{eq:t2:def}
 \Pr \left( \exists \delta \in \Delta' ( X ): \mbox{$H_{ \delta }$ not BP-rejected} \right) = \Pr \left( \cup_{ \delta \in \Delta' ( X ) } \mbox{$H_{ \delta }$ not BP-rejected} \right).
\end{equation}
Use of Boole's inequality then yields
\begin{equation}
 \Pr \left( \cup_{ \delta \in \Delta' ( X ) } \mbox{$H_{ \delta }$ not BP-rejected} \right) \leq \sum_{ \delta \in \Delta' ( X ) } \Pr \left( \mbox{$H_{ \delta }$ not BP-rejected} \right).
\end{equation}
By definition of BP, we have
\begin{equation}
 \Pr \left( \mbox{$H_{ \delta }$ not BP-rejected} \right) = \Pr \left( \mbox{$p$-value} \left( H_{ \delta } \right ) > \frac{\alpha}{m} \right),
\end{equation}
where $m$ is the number of tests. Since each individual test based on the minimum information discrimination statistic is consistent \gcite[Chap.~5, Section~5]{Kullback-1968}, the probability in the previous equation tends to 0 for $\delta \in \Delta' ( X )$ as the data size $k$ tends to infinity. As a consequence,
\begin{equation}
 \exists k_{\delta}: \ \forall k > k_{ \delta } \ \Pr \left( \mbox{$p$-value} \left( H_{ \delta } \right) > \frac{\alpha}{m} \right) < \frac{ \epsilon } { \left\vert \Delta' ( X ) \right\vert }.
\end{equation}
Therefore,
\begin{equation}
 \forall k > \max_{ \delta \in \Delta' ( X ) } \left( k_{ \delta } \right) \ \sum_{ \delta \in \Delta' ( X ) } \Pr \left( \mbox{$H_{ \delta }$ not BP-rejected} \right) < \epsilon.
\end{equation}
Inserting this back into Equation~\eqref{eq:t2:def} yields
\begin{equation} 
 \Pr \left( \exists \delta \in \Delta' ( X ): \mbox{$H_{ \delta }$ not BP-rejected} \right) < \epsilon.
\end{equation}
This entails that the probability to have type II errors tends to 0 as $k \to \infty$ or, equivalently, that the power tends to 1 as $k \to \infty$. From Equation~\eqref{eq:FDR-BP}, we finally obtain that our procedure (with FDR) is consistent as well.

\bibliographystyle{elsarticle-harv}
\bibliography{nonabrev,anglais,mabiblio}

\end{document}